\documentclass[runningheads]{llncs}
\usepackage{graphicx}
\usepackage{amsmath,amssymb} 
\usepackage{color}
\usepackage[width=122mm,left=12mm,paperwidth=146mm,height=193mm,top=12mm,paperheight=217mm]{geometry}

\usepackage{times}
\usepackage{graphicx} 
\usepackage{subfigure} 

\usepackage{algorithm}
\usepackage{algorithmic}

\usepackage{times}
\usepackage{epsfig}
\usepackage{graphicx}
\usepackage{xcolor}
\usepackage{subfiles}
\usepackage{bm}
\usepackage{multirow}
\usepackage{array}

\usepackage{times}
\usepackage{graphicx} 
\usepackage{subfigure} 

\usepackage[numbers]{natbib}
\usepackage{stmaryrd}

\usepackage{cvpr-abbriv}
\usepackage[titletoc,title]{appendix}
\usepackage{chngcntr}

\usepackage[draft=false, pagebackref=false,breaklinks=true,colorlinks,bookmarks=false,pdftex]{hyperref}
\hypersetup{hidelinks,colorlinks=false,linkbordercolor={1 1 1}}
\usepackage{thmtools, thm-restate}
\usepackage[capitalise,nameinlink]{cleveref}
\crefformat{equation}{(#2#1#3)}
\makeatletter
\let\parSym\S
\def\S{\mathcal{S}}
\makeatother
\crefname{section}{\parSym}{\parSym\parSym}
\Crefname{section}{\parSym}{\parSym\parSym}
\crefname{appendix}{\parSym}{\parSym\parSym}

\newcommand{\gray}{\color[rgb]{0.5,0.5,0.5}}
\newcommand{\red}{\color[rgb]{1,0,0}}

\newcommand{\Real}{\mathbb{R}}
\newcommand{\Y}{\mathcal{Y}}
\newcommand{\V}{\mathcal{V}}
\newcommand{\E}{\mathcal{E}}
\renewcommand{\H}{\mathcal{H}}
\newcommand{\G}{\mathcal{G}}
\newcommand{\U}{\mathcal{U}}
\newcommand{\M}{\mathcal{M}}
\renewcommand{\H}{\mathcal{H}}

\renewcommand*{\paragraph}[1]{\par\noindent{\normalsize\bf #1}\,\xspace}

\newcommand{\revisit}[1][]{%
\ifthenelse{\equal{#1}{}}{
\ensuremath{\red \triangle}\xspace}{%
{\ensuremath{\red \rhd}\xspace}%
{\gray #1}%
{\ensuremath{\red \lhd}\xspace}%
}%
}

\DeclareMathOperator*\argmin{arg\,min}

\DeclareMathOperator*\nb{Nb}

\newcommand\R{\mathbb R}
\newcommand\BR{\mathbb R}
\newcommand\SG{\mathcal G}

\newcommand\SV{\mathcal V}
\newcommand\SE{\mathcal E}

\newcommand\SM{\mathcal M}
\newcommand\SY{\mathcal Y}
\newcommand\SI{\mathcal I}

\newcommand\SU{\mathcal U}

\newcommand*{\QED}{\hfill\ensuremath{\square}}%

\newcommand\Diff{\texttt{MSD} }

\newcommand\HaSh{\texttt{MPLP}++ }

\DeclareMathAlphabet\mathbfcal{OMS}{cmsy}{b}{n}

\newcommand\ThA[2]{\theta_{#1}^{\mathcal{#2}}}

\begin{document}
\pagestyle{headings}
\mainmatter
\def\ECCV18SubNumber{2108}  

\def\mytitle{MPLP++: Fast, Parallel Dual Block-Coordinate Ascent for Dense Graphical Models}
\title{\mytitle} 


\author{Siddharth Tourani$^1$, Alexander Shekhovtsov$^2$, Carsten Rother$^1$, Bogdan Savchynskyy$^1$}

\institute{$1$ Visual Learning Lab, Uni. Heidelberg, \\$2$ Centre for Machine Perception, Czech Technical University in Prague}

\maketitle

\begin{abstract}
Dense, discrete Graphical Models with pairwise  potentials are a powerful class of models which are employed in state-of-the-art computer vision and bio-imaging applications. This work introduces a new MAP-solver, based on the popular Dual Block-Coordinate Ascent principle. Surprisingly, by making a small change to the low-performing solver, the Max Product Linear Programming (MPLP) algorithm~\cite{NIPS2007_3200}, we derive the new solver MPLP++ that significantly outperforms all existing solvers by a large margin, including the state-of-the-art solver Tree-Reweighted  Sequential  (TRW-S) message-passing algorithm~\cite{kolmogorov2006convergent}. Additionally, our solver is highly parallel, in contrast to TRW-S, which gives a further boost in performance with the proposed GPU and multi-thread CPU implementations. We verify the superiority of our algorithm on dense problems from publicly available benchmarks, as well, as a new benchmark for 6D Object Pose estimation. We also provide an ablation study with respect to graph density.
\end{abstract}

\section{Introduction}

Undirected discrete graphical models with dense neighbourhood structure are known to be much more expressive than their sparse counterparts. A striking example is the fully-connected Conditional Random Field (CRF) model with Gaussian pairwise potentials~\cite{krahenbuhl2011efficient},  significantly improving the image segmentation field, once an efficient solver for the model was proposed. More recently, various applications in computer vision and bio-imaging have  successfully used fully-connected or densely-connected, pairwise models with non-Gaussian potentials. Non-Gaussian potentials naturally arise from application-specific modelling or the necessity of robustness potentials. A prominent application of the non-Gaussian fully-connected CRF case achieved state-of-the-art performance in 6D object pose estimation problem~\cite{michel2017global}, with an efficient {\em application-specific} solver. Other examples of densely connected models were proposed in the area of stereo-reconstruction \cite{kolmogorov2006comparison}, body pose estimation~\cite{bergtholdt2010study,nowozin2011decision,kirillov2016joint}, bio-informatics~\cite{kainmueller2014active} etc.

An efficient solver is a key condition to make such expressive models efficient in practice. This work introduces such a solver, which outperforms all existing methods for a class of dense and semi-dense problems with non-Gaussian potentials. This includes Tree-Reweighted Sequential (TRW-S) message passing, which is typically used for general pairwise models. We would like to emphasize that efficient solvers for this class are highly desirable, even in the age of deep learning. The main reason is that such expressive graphical models can encode information which is often hard to learn from data, since very large training datasets are needed to learn the application specific prior knowledge. In the the above mentioned 6D object pose estimation task, the pairwise potentials encode length-consistency between the observed data and the known 3D model and the unary potentials are learned from data. Other forms of combining graphical models with CNNs such as Deep-Structured-Models~\cite{chen2015learning} can also benefit from the proposed solver.

Linear Programming (LP) relaxation is a powerful technique that can solve exactly all known tractable maximum a posteriori (MAP) inference  problems for undirected graphical models (those known to be polynomially solvable)~\cite{kolmogorov2015power}. Although there are multiple algorithms addressing the MAP inference, which we discuss in~\cref{sec:related}, the linear programs obtained by relaxing the MAP-inference problem are not any simpler than general linear programs~\cite{Prusa-PAMI-2017}. This implies that algorithms solving it exactly are bounded by the computational complexity of the general LP and do not scale well to problems of large size. Since LP relaxations need not be tight, solving it optimally is often impractical. On the other hand, block coordinate ascent (BCA) algorithms for the LP dual problem form an approach delivering fast and practically useful approximate solutions. TRW-S~\cite{kolmogorov2006convergent} is probably the most well-known and efficient solver of this class, as shown in~\cite{kappes-2015-ijcv}. Since due to the graph density the model size grows quadratically with the number of variables, a scalable solver must inevitably be highly parallelizable to be of practical use. Our work improves another well-known BCA algorithm of this type,
MPLP~\cite{NIPS2007_3200} (Max Product Linear Programming algorithm) and proposes a parallel implementation as explained next.

\paragraph{Contribution}
We present a new state-of-the-art parallel BCA algorithm of the same structure as MPLP, \ie an elementary step of our algorithm updates all dual variables related to a single graph edge.
We explore the space of such edge-wise updates and propose that a different update rule inspired by~\cite{Discrete-Continuous-16} can be employed, which significantly improves the practical performance of MPLP. Our method with the new update is termed MPLP++. The difference in the updates stems from the fact that the optimization in the selected block of variables is non-unique and the way the update utilizes the degrees of freedom to which the block objective is not sensitive significantly affects subsequent updates and thus the whole performance.

We propose the following theoretical analysis. We show that MPLP++ converges towards arc-consistency, similarly to the convergence result of~\cite{schlesingera2011diffusion} for min-sum diffusion. We further show, that given any starting point, an iteration of the MPLP++ algorithm, which processes all edges of the graph in a specified order always results in a better objective than the same iteration of MPLP. For multiple iterations this is not theoretically guaranteed, but empirically observed in all our test cases. All proofs relating to the main paper are given in the supplement.

Another important aspect that we address is parallelization. TRW-S is known as a ``sequential'' algorithm. However it admits parallelization, especially for bipartite graphs \cite{kolmogorov2006convergent}, which is exploited in specialized implementations~\cite{choi2012hardware,hurkat2015fast} for $4$-connected grid graphs. The parallelization there gives a speed-up factor of $O(n)$, where $n$ is the number of nodes in the graph. A parallel implementation for dense graphs has not been proposed.
We observe that in MPLP a group of non-incident edges can be updated in parallel. We pre-compute a schedule maximizing the number of edges that can be processed in parallel using an exact or a greedy maximum matching. The obtainable theoretical speed-up is at least $n/2$ for {\em any} graph, including dense ones.
We consider two parallel implementations, suitable for CPU and GPU architectures respectively. A further speed-up is possible by utilizing parallel algorithms for lower envelopes in message passing (see~\cref{sec:bca} and~\cite{Chen-98}).

The new MPLP++ method consistently outperforms all its competitors, including TRW-S~\cite{kolmogorov2006convergent}, in the case of densely (not necessarily fully) connected graphs, even in the sequential setting: In our experiments it is $2$ to $10$ times faster than TRW-S and $5$ to $150$ times faster than MPLP depending on the dataset and the required solution precision. The empirical comparison is conducted on several datasets. As there are only few publicly available ones, we have created a new dataset related to the 6D pose estimation problem~\cite{michel2017global}. Our code and the new benchmark dataset will be made publicly available together with the paper. 

\section{Related work}\label{sec:related}
The general MAP inference problem for discrete graphical models (formally defined in~\cref{sec:preliminaries}) is NP-hard and is also hard to approximate~\cite{Li2016a}. A natural linear programming relaxation is obtained by formulating it as a 0-1 integer linear program (ILP) and relaxing the integrality constraints~\cite{schlesinger1976syntactic} (see also the recent review~\cite{werner2007linear}). A hierarchy of relaxations is known~\cite{Zivny-Werner-Prusa-ASP-MIT2014}, from which the so-called Base LP relaxation, also considered in our work, is the simplest one. It was shown~\cite{kolmogorov2015power,Thapper-12} that this relaxation is tight for all tractable subclasses of the problem. For many other classes of problems it provides approximation guarantees, reviewed~\cite{Li2016a}. A large number of specialized algorithms were developed for this relaxation.

Apart from general LP solvers, a number of specialized algorithms exist that take advantage of the problem structure and guarantee convergence to an optimal solution of the LP relaxation. This includes proximal~\cite{Ravikumar10,MartinsICML11,MeshiGloversonECML11,SchmidtEMMCVPR11},
dual sub-gradient~\cite{storvik2000lagrangian,SchlGig_12_usim2007,komodakis2007mrf}, bundle~\cite{kappes2012bundle}, mirror-descent~\cite{luong2012solving} smoothing-based~\cite{savchynskyy2011study,savchynskyy2012efficient,meshi2012convergence} and (quasi-) Newton~\cite{kannan2017newton} methods.

However, as it was shown in~\cite{Prusa-Werner-15-Universality}, the linear programs arising from the relaxation of the MAP inference have the same computational complexity as general LPs. At the same time, if the problem is not known to belong to a tractable class, solving the relaxation to optimality may be of low practical utility. A comparative study~\cite{kappes-2015-ijcv} notes that TRW-S~\cite{kolmogorov2006convergent} is the most efficient solver for the relaxation in practice. It belongs to the class of BCA methods for the LP dual that includes also MPLP~\cite{NIPS2007_3200}, min-sum diffusion~\cite{kovalevsky1975diffusion,schlesingera2011diffusion} and DualMM~\cite{Discrete-Continuous-16} algorithms. BCA methods are not guaranteed to solve the LP dual to  optimality. They may get stuck in a suboptimal point and be unable to compute primal LP solutions unless integer solutions are found (which is not always the case). However, they scale extremely well, take advantage of fast dynamic programming techniques, solve exactly all submodular problems~\cite{schlesinger2000some} and provide good approximate solutions in general vision benchmarks.
Suboptimal solutions of the relaxation can also be employed in exact solvers~\cite{cooper2010soft, savchynskyy2013global} using cutting plane or branch-and-cut techniques and in methods identifying a part of optimal solution~\cite{Shekhovtsov-IRI-PAMI,Swoboda-PAMI-16}.

\section{Preliminaries}\label{sec:preliminaries}

\paragraph{Notation}
$\SG=(\SV,\SE)$ denotes an undirected graph, with vertex set $\SV$ (we assume $\SV = \{1,\dots |\SV|\}$) and edge set $\SE$. The notation $uv\in\SE$ will mean that $\{u,v\}\in\SE$ and $u < v$ with respect to the order of $\SV$.
Each node $u\in\SV$ is associated with a label from a finite {\em set of labels} $\SY$ (for brevity \Wlog we will assume that it is the same set for all nodes). The label space for a pair of nodes $uv\in\E$ is $\Y^2$ and for all nodes it is $\Y^\V$.

For each node and edge the {\em unary} $\theta_u : \SY \to \R$, $u\in\SV$ and {\em pairwise} cost functions $\theta_{uv} : \SY^2 \to \BR$, $uv\in\SE$ assign a cost to a label or label pair respectively. 
Let $\SI=(\V\times\Y) \cup (\E\times\Y^2)$ be the index set enumerating all labels and label pairs in neighbouring graph nodes. Let the {\em cost vector} $\theta\in\BR^{\SI}$ contain all values of the functions $\theta_u$ and $\theta_{uv}$ as its coordinates.

The {\em MAP-inference} problem for the graphical model defined by a triple $(\SG,\SY^{\SV},\theta)$ consists in finding the labelling with the smallest total cost, \ie:

\begin{equation}\label{equ:energy-min}
 y^*=\argmin_{y\in\Y^\V}\left[E(y|\theta):=\sum_{v\in\SV}\theta_v(y_v)+\sum_{uv\in\SE}\theta_{uv}(y_{uv})\right]\,.
\end{equation}

This problem is also known as {\em energy minimization} for graphical models and is closely related to weighted and valued constraint satisfaction problems. The total cost $E$ is also often called {\em energy}.

The problem~\eqref{equ:energy-min} is in general NP-hard and is also hard to approximate~\cite{Li2016a}. A number of approaches to tackle it in different practical scenarios is reviewed in~\cite{kappes-2015-ijcv,hurley2016multi}. One of the widely applicable techniques is based on (approximately) solving its linear programming (LP) relaxation as discussed in \cref{sec:related}.

\paragraph{Dual Problem}
Most existing solvers for the LP relaxation tackle its dual form, which we introduce now. This is because the LP dual has much fewer variables than the primal and can also be written in the form of unconstrained concave (piecewise-linear) maximization. 
It is based on the fact that the representation of the energy function $E(y|\theta)$ using unary $\theta_u$ and pairwise $\theta_{uv}$ costs  is not unique.  There exist other costs $\hat{\theta}\in\BR^{\SI}$ such that  $E(y|\hat{\theta})=E(y|\theta)$ for all labelings $y\in\SY^{\SV}$.  
  
It is known~(see e.g. ~\cite{werner2007linear}) and straightforward to check that such \emph{equivalent} costs can be obtained with an arbitrary vector~$\phi:=(\phi_{v \rightarrow u}(s)\in\BR \mid u\in\SV,\ v\in\nb(u),\ s\in\SY)$, where $\nb(u)$ is the set of neighbours of $u$ in $\G$, as follows:
\begin{align}\label{equ:reparametrization}
  \hat{\theta}_u(s)\equiv \theta^{\phi}_u(s) &:=\theta_u(s) + \sum\nolimits_{v\in\nb(u)}\phi_{v \rightarrow u}(s)\\
 \hat{\theta}_{uv}(s,t) \equiv \theta^{\phi}_{uv}(s,t) &:=\theta_{uv}(s,t) - \phi_{v \rightarrow u}(s) - \phi_{u \rightarrow v}(t)\,.\nonumber
\end{align}
The cost vector~$\theta^{\phi}$ is called \emph{reparametrized} and the vector~$\phi$ is known as \emph{reparametrization}.
Costs related by~\eqref{equ:reparametrization} are also called \emph{equivalent}.
Other established terms for reparametrization are \emph{equivalence preserving}~\cite{cooper2004arc} or \emph{equivalent transformations}~\cite{schlesinger1976syntactic}. 

By swapping $\min$ and $\sum$ operations in~\eqref{equ:energy-min} one obtains a lower bound on the energy $D(\theta^{\phi}) \le E(y|\theta)$ for all $y$, which reads
\begin{equation}\label{equ:LP-lower-bound}
  D(\theta^{\phi}):=\sum_{u\in\SV}\min_{s\in\SY}\theta^{\phi}_u(s) + \sum_{uv\in\SE}\min_{(s,t)\in\SY^2}\theta^{\phi}_{uv}(s,t)\,.
\end{equation}
Although the energy $E(y|\theta)$ remains the same for all equivalent cost vectors (\eg $E(y|\theta)=E(y|\theta^{\phi})$), the lower bound $D(\theta^{\phi})$ depends on reparametrization, which is~ $D(\theta)\neq D(\theta^{\phi})$.
Therefore, a natural maximization problem arises as maximization of the lower bound over all equivalent costs: $\max_{\phi}D(\theta^{\phi})$.
It is known~(e.g. \cite{werner2007linear}) that this maximization problem can be seen as a dual formulation of the LP relaxation of~\eqref{equ:energy-min}. 
We will write $D(\phi)$ to denote $D(\theta^{\phi})$, since the cost vector $\theta$ is clear from the context.The function $D(\phi)$ is concave, piece-wise linear and therefore non-smooth. In many applications the dimensionality of $\phi$ often exceeds $10^5$ to $10^6$ and the respective dual problem $\max_{\phi} D(\phi)$ is large scale. 

\section{Dual Block-Coordinate Ascent}\label{sec:bca}
As we discussed in \cref{sec:related}, BCA methods, although not guaranteed to solve the dual to the optimality, provide good solutions for many practical instances and scale very well to large problems. The fastest such methods are represented by methods working with chain subproblems~\cite{kolmogorov2006convergent}, \cite{Discrete-Continuous-16} or their generalizations~\cite{kolmogorov2015new}.

The TRW-S algorithm can be seen as updating a block of dual variables $(\phi_{v \rightarrow u},\ v\in\nb(u))$ ``attached'' to a node $u$, during each elementary step. The same block of variables is also used in the {\em min-sum diffusion} algorithm~\cite{schlesingera2011diffusion}, as well as in {\em the convex message passing}~\cite{Hazan08norm-productbelief}, and~\cite{kolmogorov2015new} gives a generalization of such methods. 

However, the update coefficients in TRW-S are related to the density of the graph and its advantage diminishes when the graph becomes dense.
We show that for dense graphs updating a block of dual variables $\phi_{u \leftrightarrow v} = (\phi_{v \rightarrow u}, \phi_{u \rightarrow v})$ associated to an edge $uv\in\SE$ can be more efficient. Such updates were previously used in the MPLP algorithm~\cite{NIPS2007_3200}. We show that our MPLP++ updates differ in detail but bring a significant improvement in performance.

\subsubsection{Block Optimality Condition}
For further analysis of BCA algorithms we will require a sufficient condition of optimality \wrt the selected block of dual variables. The restriction of the dual to the block of variables $\phi_{u \leftrightarrow v}$ is given by the function:
\begin{equation}\label{restricted-dual}
D_{uv}(\phi_{u \leftrightarrow v}):= \min\limits_{st\in\SY^2}\theta^{\phi}_{uv}(s,t)+ \min\limits_{s\in\SY}\theta^{\phi}_{u}(s)+\min\limits_{t\in\SY}\theta^{\phi}_{v}(t)\,,  
\end{equation}
Maximizing $D_{uv}$ is equivalent to performing a BCA \wrt $\phi_{u \leftrightarrow v}$ for $D(\phi)$. 
The necessary and sufficient condition of maximum of $D_{uv}$ are given by the following.

\begin{restatable}{proposition}{TblockOpt}\label{prop:block-optimality-condition}
Reparametrization $\phi_{u \leftrightarrow v}$ maximizes $D_{uv}(\cdot)$ iff there exist $(s,t)\in\Y^2$ such that
$s$ minimizes $\theta^{\phi}_{u}(\cdot)$, $t$ minimizes $\theta^{\phi}_{v}(\cdot)$ and
$(s,t)$ minimizes $\theta^{\phi}_{uv}(\cdot,\cdot)$.
\end{restatable}

This condition is trivial to check, it is a special case of arc consistency~\cite{werner2007linear} or weak tree-agreement~\cite{kolmogorov2006convergent} when considering a simple graph with one edge.
It is also clear that $\phi_{u \leftrightarrow v}$ satisfying this condition is not unique. For BCA algorithms it means that there are degrees of freedom in the block, which do not directly affect the objective. By moving on a plateau, they can nevertheless affect subsequent BCA updates. Therefore the performance of the algorithm will be very much dependent on the particular BCA update rule satisfying~\cref{prop:block-optimality-condition}.

\subsubsection{Block Coordinate Ascent Updates}
Given the form of the restricted dual~\eqref{restricted-dual} on the edge $uv$, all BCA-updates can be described as follows. Assume $\theta$ is the current reparametrized cost vector, \ie $\theta = \bar\theta^{\bar\phi}$ for some initial $\bar \theta$ and the current reparametrization $\bar \phi$.
\begin{definition}\label{def:BCA-update}
A {\em BCA update} takes on the input an edge $uv\in\E$ and costs $\theta_{uv}(s,t)$, $\theta_u(s)$ and $\theta_v(t)$ and outputs a reparametrization $\phi_{u\leftrightarrow v}$ satisfying \cref{prop:block-optimality-condition}. \WLOG, we assume that it will also satisfy $\min_{s,t}\theta^\phi_{uv}(s,t) = 0$.\footnote{This fixes the ambiguity \wrt a constant that can be otherwise added to the edge potential and subtracted from one of the unary potentials. This constant does not affect the performance of algorithms. }
\end{definition}

According to~\eqref{equ:reparametrization} a BCA-update results in the following reparametrized potentials:
\begin{align}
\theta^{\phi}_u = \theta_u +\phi_{v\to u}, \ \ \ \theta^{\phi}_v = \theta_v +\phi_{u\to v}, \ \ , \theta^{\phi}_{uv}(s,t) = \theta_{uv} - \phi_{v\to u} - \phi_{u\to v}.
\end{align}
Note that since all reparametrizations constitute a vector space, after a BCA-update $\phi$ we can update the current total reparametrization as $\bar\phi := \bar\phi + \phi$.

We will consider BCA-updates of the following form: first construct the aggregated cost $g_{uv}(s,t) = \theta_{uv}(s,t) + \theta_u(s) + \theta_v(t)$. This corresponds to applying a reparametrization $\mathring\phi_{u\leftrightarrow v} = (-\theta_u, -\theta_v)$, which gives $\theta_{uv}^{\mathring\phi} = g_{uv}$, $\theta_{u}^{\mathring\phi} = \theta_{v}^{\mathring\phi} = 0$. After that, a BCA update forms a new reparametrization $\phi_{u\leftrightarrow v}$ such that $\theta^{\mathring\phi + \phi}$ satisfies~\cref{prop:block-optimality-condition}.

Such BCA-updates can be represented in the following form, which will be simpler for defining and analysing the algorithms.

\begin{definition}Consider a BCA-update using a composite reparametrization $\mathring\phi_{u\leftrightarrow v} + \phi_{u\leftrightarrow v}$. It can be then fully described by the {\em reparametrization mapping} $\gamma \colon g_{uv} \to (\theta^\gamma_u, \theta^\gamma_v)$, 
where $g_{uv} \in \Real^{\Y_{uv}}$, $\theta^\gamma_u = \phi_{v\rightarrow u}$ and $\theta^\gamma_v = \phi_{u\rightarrow v}$.
\end{definition}

\begin{figure*}[t]
\centering
\includegraphics[width=12cm]{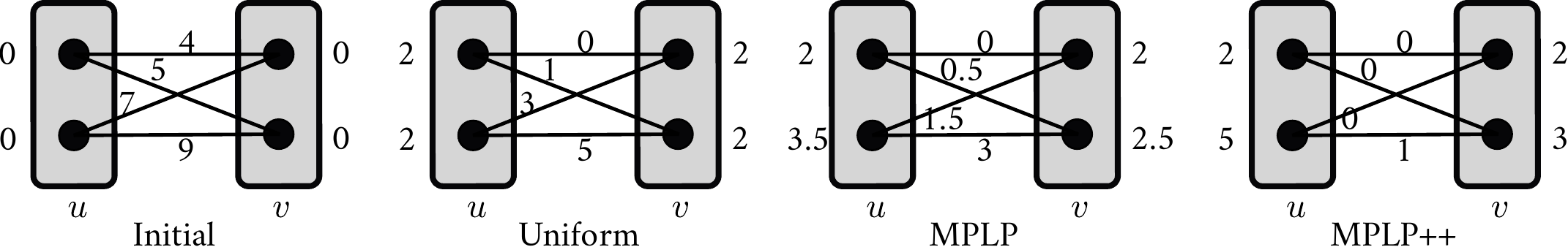}
\caption{Illustration of the considered BCA-updates. The gray boxes in the figure represent graph nodes. The black dots in them the labels. The edges connecting the black dots make up the pairwise costs. The numbers adjacent to the edges and labels are the pairwise and unary costs respectively.}
\label{fig:algs-illus}
\end{figure*}

By construction, $\theta^\gamma_u$ matches the reparametrized unary term $\theta^{\mathring\phi+\phi}_u$, $\theta^\gamma_v$ is alike and the reparametrized pairwise term is given by $\theta^\gamma_{uv} = g_{uv} - \phi_{v\rightarrow u} - \phi_{u\rightarrow v} = g_{uv} -\theta^\gamma_{u} -\theta^\gamma_{v}$. 
In what follows, BCA-update will mean specifically the reparametrization mapping $\gamma$.
We define now several BCA-updates that will be studied further.

$\bullet$ The {\tt uniform} BCA-update is given by the following reparametrization mapping~$\U$:

\begin{equation}\label{equ:uniform-update}
\tag{$\U$}
\begin{aligned}
 \theta^\U_u(s) & = \theta^\U_v(t) := {\textstyle\frac{1}{2}} \min\nolimits_{s',t'\in\SY}g_{uv}(s', t'),\ \forall s,t\in\Y.
\end{aligned}
\end{equation}
This is indeed just an example, to illustrate the problem of non-uniqueness of the minimizer. It is easy to see that this update satisfies~\cref{prop:block-optimality-condition} since both $\theta^\U_u(\cdot)$ and $\theta^\U_v(\cdot)$ are constant and therefore any pairwise minimizer of $\theta^\U_{uv}$ is consistent with them.

$\bullet$ The {\tt MPLP} BCA-update is given by the following reparametrization mapping~$\M$:
\begin{equation}\label{equ:MPLP-update}
\tag{$\M$}
\begin{aligned}
\theta^\M_u(s) & :=\textstyle \frac{1}{2} \min_{t \in \Y}g_{uv}(s,t),\ \forall s \in \SY,\\
\notag
\theta^\M_v(t) & :=\textstyle \frac{1}{2} \min_{s \in \Y}g_{uv}(s,t),\ \forall t \in \SY.
\end{aligned}
\end{equation}
The MPLP algorithm~\cite{NIPS2007_3200} can now be described as performing iterations by applying BCA-update $\M$ to all edges of the graph in a sequence.

$\bullet$ The new {\tt MPLP++} BCA-update, that we propose, is based on the {\em handshake} operation~\cite{Discrete-Continuous-16}. It is given by the following procedure defining the reparametrization mapping~$\H$:
\begin{equation}\label{equ:handshake-update}
\tag{$\H$}
\begin{aligned}
\theta^\H_u(s) & \textstyle := \theta^\M_u(s),\quad \theta^\H_v(s) \textstyle := \theta^\M_v(s),\ \forall s \in \SY\,, \\
\notag
\theta^\H_v(t) & \textstyle := \theta^\H_v(t) + \min_{s \in \Y}[g_{uv}(s,t) - \theta^\H_v(t)- \theta^\H_u(s)],\ \forall t \in \SY\,,\\
\notag
\theta^\H_u(s) & \textstyle := \theta^\H_u(s) + \min_{t \in \Y}[g_{uv}(s,t) - \theta^\H_v(t)- \theta^\H_u(s)],\ \forall s \in \SY\,.
\end{aligned}
\end{equation}

In other words, the {\tt MPLP++} update first performs the {\tt MPLP} update and then pushes as much cost from the pairwise factor to the unary ones, as needed to fulfill
$\min_{t}\theta^{\M}_{uv}(s,t)=\min_{s}\theta^{\M}_{uv}(s,t)=0$ for all labels $s$ and $t$ in the nodes $u$ and $v$ respectively.
It is also easy to see that the assignment $\theta^\H_v(s) := \theta^\M_v(s)$ together with the second line in~\ref{equ:handshake-update} can be equivalently substituted by
$\theta^\H_v(t) := \min_{s \in \Y}[g_{uv}(s,t) - \theta^\H_u(s)],\ \forall t \in \SY$. This allows to perform the {\tt MPLP++} update with $3$ minimizations over $\SY^2$ instead of $4$. \cref{fig:algs-illus} shows the result of applying the three BCA-updates on a simple two-node graph.

It is straightforward to show that $\M$ and $\H$ also satisfy~\cref{def:BCA-update} (see supplement) and therefore are liable BCA-updates. In spite of that, the behavior of all three updates is notably different, as it is shown by example in \cref{fig:strong-potts}. Therefore, proving only that some algorithm is a BCA, does not imply its efficiency.

\begin{figure} [!t]
\centering
 \subfigure[Uniform]{\label{fig:bca-uniform} \includegraphics[width=0.30\linewidth]{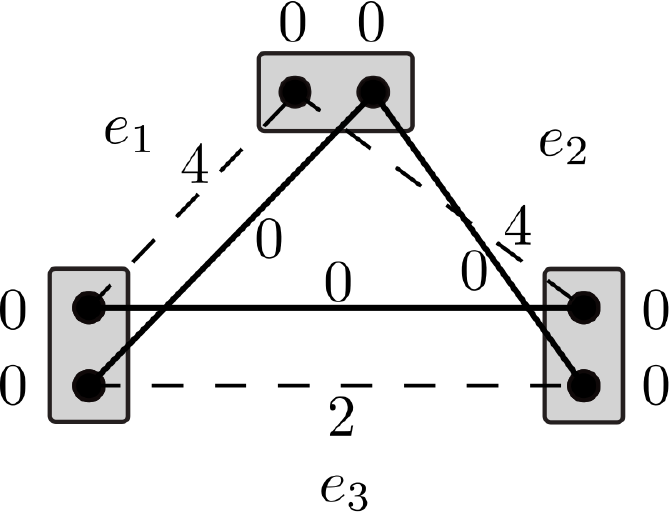}\quad}
 \subfigure[MPLP]{\label{fig:bca-mplp} \includegraphics[width=0.30\linewidth]{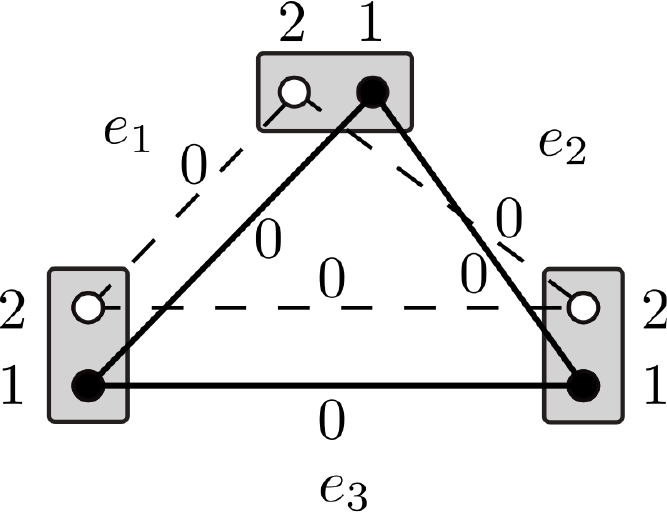}}
\caption{
 Choosing the right BCA-update is important. Notation has the same meaning as in Fig.~\ref{fig:algs-illus}, black circles denote locally optimal labels, white circles - the non-optimal ones.
 Solid lines correspond to the locally optimal pairwise costs connected to the locally optimal labels. Omitted lines in pairwise interactions denote infinite pairwise costs. $e_i$ denotes edge indexes, edge processing order is according to the subscript $i$.
 \textbf{(a)} Uniform update gets stuck and is unable to optimize the dual further. \textbf{(b)} {\tt MPLP} and {\tt MPLP++} attain the dual optimum in one iteration.
}
\label{fig:strong-potts}
\end{figure}

\subsubsection{Message Passing}
Importantly for performance, updates $\U$, $\M$ and $\H$ can be computed using a subroutine computing $\min_{t}[\theta_{uv}(s, t) + a(t)]$ for all $s$, where $\theta$ is a fixed initial pairwise potential and $a$ is an arbitrary input unary function. This operation occurring in dynamic programming and all of the discussed BCA algorithms, is known as {\em message passing}. In many cases of practical interest it can be implemented in time $O(|\Y|)$ (e.g. for Potts, absolute difference and quadratic costs) rather than $O(|\Y|^2)$ in the general case, using efficient sequential algorithms, \eg,~\cite{Aggarwal1987}. Using $|\Y|$ processors, the computation time can be further reduced to $O(\log |\Y|)$ with appropriate parallel algorithms, \eg,~\cite{Chen-98}.
\subsubsection{Primal Rounding} BCA-algorithms iterating BCA-updates give only a lower bound to the MAP-inference problem~\eqref{equ:energy-min}. To obtain a primal solution, we use a sequential rounding procedure similar to the one proposed in~\cite{kolmogorov2006convergent}. Assuming we have already computed a primal integer solution $x^{*}_v$ for all $v < u$, we want to compute $x^{*}_u$. To do so, we use the following equation for the assignment
\begin{equation}
\textstyle x^{*}_u \in \argmin_{x_u\in\Y} \Big[ \theta_u(x_u)+\sum_{v<u \mid uv \in \mathcal{E}} \theta_{uv}(x_u,x^{*}_v) \Big],
\end{equation}
where $\theta$ is the reparametrized potential produced by the algorithm.

\section{Theoretical Analysis}

As we prove below, {\tt MPLP++} in the limit guarantees to fulfill a necessary optimality condition, related to the {\em arc-consistency}~\cite{werner2007linear} and the {\em weak tree-agreement}~\cite{kolmogorov2006convergent}.

\paragraph{Arc-Consistency}
Let $\llbracket\cdot\rrbracket$ be the Iverson bracket, \ie $\llbracket A\rrbracket=1$ if $A$ holds. Otherwise $\llbracket A\rrbracket=0$.
Let $\bar\theta_u(s):=\llbracket \theta_u(s)=\min_{s'}\theta_u(s') \rrbracket$ and $\bar\theta_{uv}(s,t):=\llbracket \theta_{uv}(s,t)=\min_{s',t'}\theta_{uv}(s',t') \rrbracket$ be binary vectors
with values $1$ assigned to the locally minimal labels and label pairs. Let also logical {\em and} and {\em or} operations be denoted as $\wedge$ and $\vee$. To the binary vectors they apply coordinate-wise. 

\begin{definition}
 A vector $\bar\theta\in\{0,1\}^{\SI}$ is called {\em arc-consistent}, if $\bigvee_{t\in\SY}\bar\theta_{uv}(s,t)=\bar\theta_u(s)$ for all $\{u,v\}\in\SE$ and $s\in\SY$.  
\end{definition}

However, arc-consistency itself is not necessary for dual optimality. The necessary condition is existence of the node-edge agreement, which is a special case of the weak tree agreement~\cite{kolmogorov2006convergent} when individual nodes and edges are considered as the trees in a problem decomposition. This condition is also known as a non-empty {\em  kernel}~\cite{werner2007linear} / {\em arc-consistent closure}~\cite{cooper2004arc} of the cost vector $\theta$.

\begin{definition}
 We will say that the costs $\theta\in\BR^{\SI}$ fulfill the {\em node-edge agreement}, if there is an arc-consistent vector $\xi\in\{0,1\}^{\SI}$ such that $\xi\wedge\bar\theta=\xi$.   
\end{definition}

\subsubsection{Convergence of MPLP++} It is clear that all BCA algorithms are monotonous and converge in the dual objective value as soon as the dual is bounded (\ie, primal is feasible). However, such convergence is weaker than convergence in the reparametrization than is desired. To this end we were able to show something in between the two: the convergence of {\tt MPLP++} in a measure quantifying violation of the node-edge agreement, a result analogous to~\cite{schlesingera2011diffusion}. 

\begin{restatable}{theorem}{TconvergenceAC}\label{thm:alg-convergence}
The {\tt MPLP++} algorithm converges to node-edge agreement.
\end{restatable}

Another important question is the comparison of different BCA methods.
When comparing different algorithms, an ultimate goal is to prove faster convergence of one compared to the other one.
We cannot show that the new {\tt MPLP++} algorithm has a better theoretical convergence rate. First, such rates are generally unknown for BCA algorithms for non-smooth functions. Second, considered algorithms are all of the same family and it is likely that their asymptotic rates are the same. Instead, we study the {\em dominance}, the condition that allows to rule out BCA updates which are always inferior to others. Towards this end we show that given the same starting dual point, one iteration of {\tt MPLP++} always results in a better objective value than that of {\tt MPLP} and {\tt uniform} BCA. While this argument does not extend theoretically to multiple iterations of each method, we show that it is still true in practice for all used datasets and a significant speed-up (up to two orders) is observed. The experimental comparison in~\cref{sec:experiments} gives results in wall-clock time as well as in a machine-independent count of the message passing updates performed.

\subsection{Analysis of BCA-updates}
\begin{definition}A {\em BCA-iteration} $\alpha$ is defined by the BCA-update $\gamma$ applied to all edges $\E$ in some chosen order.
Let also $D(\alpha)$ represent the dual objective value with the reparametrization defined by the iteration $\alpha$ on the input costs $\theta$.
\end{definition}
We will analyze BCA-iterations of different BCA-updates \wrt the same sequence of edges. The goal is to show that an iteration of the new update dominates the baselines in the dual objective. This property is formally captured by the following definition.

 \begin{definition}
 We will say that a BCA-iteration $\alpha$ {\em dominates} a BCA-iteration $\beta$, if for any input costs it holds $D(\alpha)\ge D(\beta)$. 
\end{definition}

In order to show it, we introduce now and prove later the dominance relations of individual BCA-updates. They are defined not on the dual objective but on all unary components.
\begin{definition}\label{def:dominance}
Let $\gamma$ and  $\delta$ be two BCA-updates.
We will say that update $\gamma$ {\em dominates} $\delta$ (denoted as $\gamma\ge \delta$) if for any $g_{uv} \in\BR^{\Y^2}$ it holds that $\gamma[g_{uv}] \geq \delta[g_{uv}]$, where the inequality is understood as component-wise inequalities $\theta^\gamma_u[g_{uv}] \geq \theta^\delta_u[g_{uv}]$ and $\theta^\gamma_v[g_{uv}] \geq \theta^\delta_v[g_{uv}]$.
\end{definition}

We can show the following dominance results. Recall that $\mathcal{U}$, $\mathcal{H}$ and $\mathcal{M}$ are the {\tt uniform}, {\tt MPLP} and {\tt MPLP++} BCA-updates, respectively.
\begin{restatable}{proposition}{Pbcadominances}\label{thm:hs-update-doms}
The following BCA-dominances hold: $\mathcal{H} \geq \mathcal{M} \geq \mathcal{U}$.
\end{restatable}
It is easy to see that the dominance~\cref{def:dominance} is transitive and so also $\mathcal{H} \geq \mathcal{U}$.

We will prove that such coordinate-wise dominance of BCA-updates implies also the dominance in the dual objective whenever the following monotonicity property holds:
\begin{definition}
A BCA-update $\gamma$ is called {\em monotonous} if 
$(\theta_u \geq \theta'_u$,\ $\theta_v \geq \theta'_v)$ implies $\gamma[\theta_u + \theta_{uv} + \theta_v] \ge \gamma[\theta'_u + \theta_{uv} + \theta'_v]$ for all $\theta, \theta'$.
\end{definition}

\begin{restatable}{proposition}{Pmonotone}\label{thm:mplp-is-monotonous}
Updates $\mathcal{U}$ and $\mathcal{M}$ are monotonous. The update $\mathcal{H}$ is not monotonous.
\end{restatable}

With these results we can formulate our main claim about domination in the objective value for the whole iteration.

\begin{restatable}{theorem}{Tdominance}\label{thm:algorithm-domination}
Let BCA-update $\gamma$ dominate BCA-update $\mu$ and let $\mu$ be monotonous. Then a BCA-iteration with $\gamma$ dominates a BCA-iteration with $\mu$.
\end{restatable}
From \cref{thm:hs-update-doms}, \cref{thm:mplp-is-monotonous} and \cref{thm:algorithm-domination} it follows now that BCA-iteration of {\tt MPLP++} dominates that of {\tt MPLP}, which in its turn dominates uniform.

\section{Parallelization}
\label{sec:parallelization}
To optimize $D(\theta^{\phi})$, we have to perform local operations on a graph, that per reparametrization influence only one edge $uv \in \mathcal{E}$ and it's incident vertices $u$ and $v$. The remaining graph $\mathcal{G}'=(\mathcal{V}-\{u,v\},\mathcal{E}-\{I_u \cup I_v\})$ remains unchanged. This gives rise to opportunities for parallelization. However, special care has to be taken to prevent {\em race conditions} which occurs when two or more threads access shared data and they try to change it at the same time.
 
Consider the case of~\cref{fig:par-details}, choosing edges $1$ and $2$ or $1$ and $6$ to process in parallel. These edges have vertex A in common, which would lead to race conditions. Processing edges $1$ and $3$ in parallel would lead to more parallelization as there are no conflicting nodes. Thus, for maximal parallelization we have to come up with an ordering of edges where threads working in parallel do not have common vertices.

Finding such edges without intersecting vertices is a well-studied problem in combinatorial optimization \cite{schrijver2003combinatorial}. A \emph{matching} $\M \subset \E$ in graph $\mathcal{G}$ is a set of edges such that no two edges in $\M$ share a common vertex. A matching is \emph{maximum} if it includes the largest number of edges, $|\M|$. Every edge in a matching can be processed in parallel without race conditions ensuing. There exist efficient greedy algorithms to find a maximum matching which we use. This gives rise to Algorithm~\ref{alg:edge-sched} for covering all edges of the graph while ensuring good parallelization. To cover the entire graph, we call a matching algorithm repeatedly, until all edges are exhausted. 

Initially, in line 1 the edge queue $\mathcal{Q_E}$ is empty. In line $3$, a maximum matching $\mathcal{E_M}$ is found. This is added to $\mathcal{Q_E}$ in line 4. This continues until all edges have been exhausted, \ie the edges remaining $\mathcal{E_R}$ is empty. The queue thus has a structure $\mathcal{Q_E}=(\mathcal{E}^1_\SM,\mathcal{E}^2_{\SM}, ..., \mathcal{E}^n_\SM)$, ordered left to right. $\mathcal{E}^i_\SM$ being the $i^{th}$  matching computed. The threads running in parallel can keep popping edges from $\mathcal{Q_E}$ and processing them without much need for mutex locking.

We have different implementation algorithms for GPUs and multi-core CPUs. 

\begin{figure*}[!t]
\centering
\begin{subfigure}[Parallelization Illustration]{
\includegraphics[scale=0.013]{./plots/parallelization_details}}
\end{subfigure}
\begin{subfigure}[GPU vs CPU Performance]{
\includegraphics[scale=0.30]{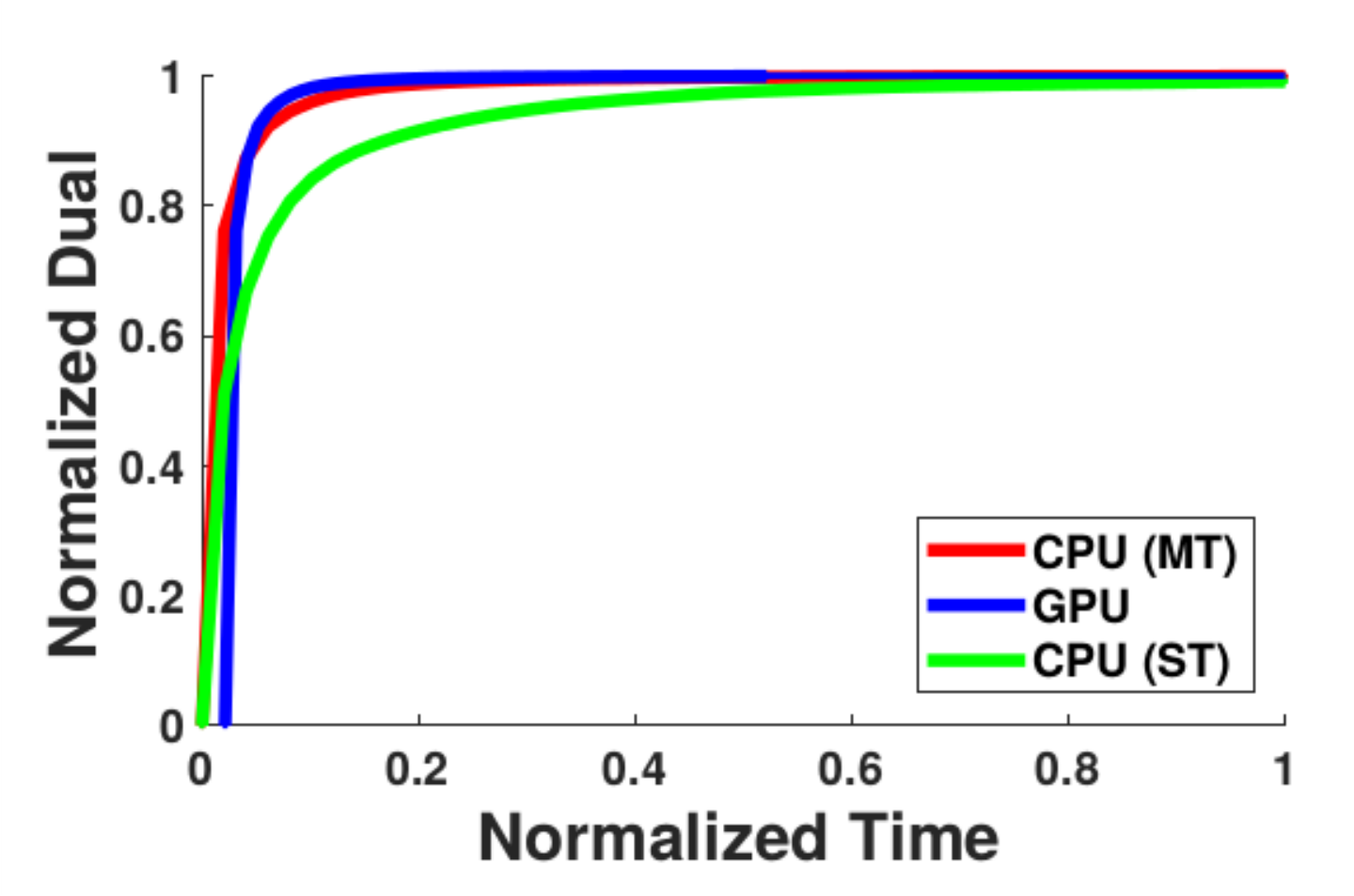}
}
\end{subfigure}
\caption{Figure shows parallelization  details and performance comparison for CPU and GPU. (a) shows the details of how the edge schedule is computed for maximizing throughput. The first row shows how edge selection is carried out by finding matchings and adding the edges in these matchings to the queue. The second row shows how the threads are launched for the CPU and GPU. For the CPU, threads are launched dynamically at different time instances, and no synchronization is carried out across all threads. This is due to a local memory locking mechanism (mutex) for the CPU. For the GPU, many threads are launched simultaneously and synchronized simultaneously via a memory barrier. This barrier is shown as the vertical line with \textbf{Synch.} (b) shows the running time comparison between single-threaded, multi-threaded and GPU versions of the \HaSh algorithm. The GPU takes some time to load memory from the host (CPU) to device (GPU). This is why it takes longer to get started than the CPU versions.} 
\label{fig:par-details}
\end{figure*} 

\begin{algorithm}[!h] 
\caption{Compute Edge Schedule} 
\label{alg1} 
\begin{algorithmic}[1] 
    \REQUIRE $\mathcal{G}=(\mathcal{V},\mathcal{E})$
    \STATE \textbf{Initial}: $\mathcal{Q_{E}}:=\emptyset$ \ \ \ \ \ ({\small Empty edge queue}), \newline $\mathcal{E_R}:=\mathcal{E}$ \ \ \ \ \ ({\small Initial pool of edges})
	\WHILE{$\mathcal{E_R}!=\emptyset$}
	   \STATE $\mathcal{E_M}:=\textrm{Maximum\_Matching}(\mathcal{V},\mathcal{E_R})$
    \STATE $\mathcal{Q_E}.push(\mathcal{E_M})$ \ \ \ \ \ ({\small Push maximum-matching  $\mathcal{E_M}$ to the queue})
    	\STATE $\mathcal{E_R}:=\mathcal{E_R}-\mathcal{E_M}$ \ \ \ \ \ ({\small Remove matched edges from $\mathcal{E_R}$})
\ENDWHILE
\end{algorithmic}
\label{alg:edge-sched}
\end{algorithm}

\textbf{CPU Implementation: } Modern CPUs consist of multiple cores with each core having one hardware threads. Hyper-threading allows for an additional thread per core, but with lesser performance in the second thread compared to the first. 

Processing an edge is a short-lived task and launching a separate thread for each edge would have excessive overhead. To process lightweight tasks we use the \emph{thread-pool} design pattern. A thread-pool keeps multiple threads waiting for tasks to be executed concurrently by a supervising program. The threads are launched only once and continuously process tasks from a task-queue. As they are launched only once, the latency in execution due to overhead in thread creation and destruction is avoided.

In the case of our algorithm the task queue is $\mathcal{Q_E}$. The thread picks up the index of the edge to process and performs the {\tt MPLP++} operation~\cref{equ:handshake-update}. During the {\tt MPLP++} operation the node and edge structures are locked by mutexes. One iteration of the algorithm is complete when all edges have been processed. Since multiple iterations maybe required to reach convergence, our task queue is circular, letting the threads restart the reparameterization process from the first element of $\mathcal{Q_E}$. The ordering of $\mathcal{Q_E}$ prevents heavy lock contention of mutexes.

\textbf{GPU Implementation: } Unlike CPUs, GPUs do not use mutexes for synchronization. The threads in each GPU processor are synchronized via a hardware barrier synchronization. A barrier for a group of threads stops the threads at this point and prevents them from proceeding until all other threads/processes reach this barrier. 

This is where the ordering of $\mathcal{Q_E}$ comes handy. Recall the structure of 
$\mathcal{Q_E}=(\mathcal{E}^1_\SM,\mathcal{E}^2_{\SM}, ..., \mathcal{E}^n_\SM)$.
Barrier synchronization can be used  between the matchings $\mathcal{E}^i_\SM$ and $\mathcal{E}^{i+1}_\SM$,  
allowing for the completion of the {\tt MPLP++} BCA update operations for $\mathcal{E}^i_\SM$, before beginning the processing of $\mathcal{E}^{i+1}_\SM$. This minimizes the time threads spend waiting while the other threads complete, as they have no overlapping areas to write to in the memory.

\section{Experimental Evaluation}\label{sec:experiments}

In our experiments, we use a $4$-core Intel i7-4790K CPU @ 4.00GHz, with hyperthreading, giving $8$ logical cores.
For GPU experiments, we used the NVIDIA Tesla K80 GPU with $4992$  cores. 

\paragraph{Compared Algorithms} 
We compare different algorithms in two regimes: the wall-clock time and the machine-independent regime, where we count the number of operations performed. 
For the latter one we use the notion of the {\em oracle call}, which is an operation like $\min_{t \in \Y}g_{uv}(s,t),\ \forall s \in \SY$ involving single evaluation of all costs of a pairwise factor.
When speaking about {\em oracle complexity} we mean the number of oracle calls per single iteration of the algorithm. As different algorithms have different oracle complexities we define a {\em normalized iteration} as exactly $|\SE|$ messages for an even comparison across algorithms. We compare the following \texttt{BCA} schemes:
\begin{itemize}
\item The Tree-Reweighted Sequential ({\tt TRWS}) message processing ~\cite{kolmogorov2006convergent} algorithm has consistently been the best performing method on several benchmarks like ~\cite{kappes-2015-ijcv}. Its oracle complexity is $2|\mathcal{E}|$.  We use the multicore implementation introduced  in~\cite{Shekhovtsov-IRI-PAMI}
for comparison, which is denoted as \texttt{TRWS(MT)} when run on multiple cores.
\item The Max-Product Linear Programming~{\tt MPLP}~\cite{NIPS2007_3200} algorithm with BCA-updates~\eqref{equ:MPLP-update}. It's oracle complexity is thus $2|\mathcal{E}|$. For~{\tt MPLP} we have our own multi-threaded implementation that is faster than the original one by a factor of $4$
\item The Min-Sum Diffusion Algorithm~\Diff~\cite{schlesingera2011diffusion}, is one of the earliest (in the 70s) proposed \texttt{BCA} algorithms for graphical models. The oracle complexity of the method is~$4|\mathcal{E}|$.
\item The {\tt MPLP++} algorithm with BCA-updates~\eqref{equ:handshake-update} has the oracle complexity of~$3|\mathcal{E}|$. The algorithm is parallelized as described in~\cref{sec:parallelization} for both CPU and GPU. The corresponding legends are {\tt MPLP++(MT)} and {\tt MPLP++(GPU)}.
\end{itemize}



\paragraph{Dense datasets} To show the strength of our method we consider the following datasets with underlying densely connected graphs:
\begin{itemize}
\item \texttt{worms} dataset~\cite{dataset-worm} consists of $30$ problem instances coming from the field of bioimaging. The problems' graphs are dense but not fully connected with about $0.1\cdot\SV^2$ of edges, up to $600$ nodes and up to $1500$ labels.
\item \texttt{protein}-folding dataset~\cite{dataset-proteinfolding} taken from the OpenGM benchmark~\cite{kappes-2015-ijcv}. This dataset has $21$ problem instances with $33-1972$ variables. The models are fully connected and have $81-503$ labels per node.
\item \texttt{pose} is the dataset inspired by the recent work~\cite{michel2017global} showing state-of-the-art performance on the $6$D object pose estimation problem.
In~\cite{michel2017global} this problem is formulated as MAP-inference in a fully connected graphical model. The set of nodes of the underlying graph coincides with the set of pixels of the input image (up to some downscale factor), which requires specialized heuristics to obtain practically useful solutions in reasonable time due to the large problem size. Contrary to the original work, we assume that the position of the object
is given by its mask (we used ground-truth data from the validation set, but assume the mask could be provided by some segmentation method) and treat only the pixels inside the mask as the nodes of the corresponding graphical model. Otherwise, the unary and pairwise costs are constructed in the same way as in~\cite{michel2017global}. We have got the learned unary costs from the authors of~\cite{michel2017global} and merely tuned the hyper-parameters on the validation set. This dataset has $32$ problem instances with $600-4800$ variables each and $13$ labels per node. The models are all fully connected. 
\end{itemize}

\paragraph{Sparse datasets} Although our method is not the best performing one on sparse graphical models, we include the comparison on the following four-connected grid-graph based benchmark datasets for fairness:
\begin{itemize}
\item \texttt{color-seg} from~\cite{kappes-2015-ijcv} has Potts pairwise costs. The nodes contain up to $12$ labels.
\item \texttt{stereo} from the Middlebury MRF benchmark~\cite{szeliski2008comparative}. This dataset consists of $3$ models with truncated linear pairwise costs and $16$, $20$ and $60$ labels respectively. 
\end{itemize}

\paragraph{Algorithm Convergence} ~\cref{fig:dual-vs-iters} shows convergence of the considered algorithms in a sequential setting for the {\tt protein} and {\tt stereo} datasets as representatives of the dense and sparse problems. For other datasets the behavior is similar, therefore we moved the corresponding plots to the supplement~\cref{sec:extra-exp}. The proposed {\tt MPLP++} method outperforms all its competitors on {\em all}  dense problem instances and is inferior to {\tt TRWS} only on sparse ones. This holds for both comparisons: the implementation-independent {\em normalized iteration} and the implementation-dependent {\em running time} ones. 
~\cref{fig:speedups} shows relative time speed-ups of the considered methods as a function of the attained dual precision. The speed of {\tt MPLP}, as typically the slowest one, is taken to be $1$, \ie, for other algorithms the speed-up compared to {\tt MPLP} is plotted. This includes also the CPU-parallel versions of {\tt MPLP++}, {\tt TRWS} and the GPU-parallel {\tt MPLP++}. \cref{fig:speedups} also shows that for dense models {\tt MPLP++} is $2-10$ faster than {\tt TRWS} in the sequential setting and $7-40$ times in the parallel setting. The speed-up w.r.t.\ {\tt MPLP} is $5-150$ times depending on the dataset and the required precision.

\paragraph{Performance Degradation with Graph Sparsification} In this test we gradually and randomly removed edges from the graphs of the {\tt pose} dataset and measures performance of all algorithms. ~\cref{fig:perf-deg-time} shows that up to at least $10\%$ of all possible edges {\tt MPLP++} leads the board. Only when the number of edges drops to $5\%$ and less, {\tt TRWS} starts to outperform {\tt MPLP++}. Note, the density of edges in grid graphs considered above does not exceed the $0.007\%$ level.


\begin{figure}
\centering
\begin{subfigure}[{\tt protein}, iter]{\includegraphics[width=0.23\linewidth]{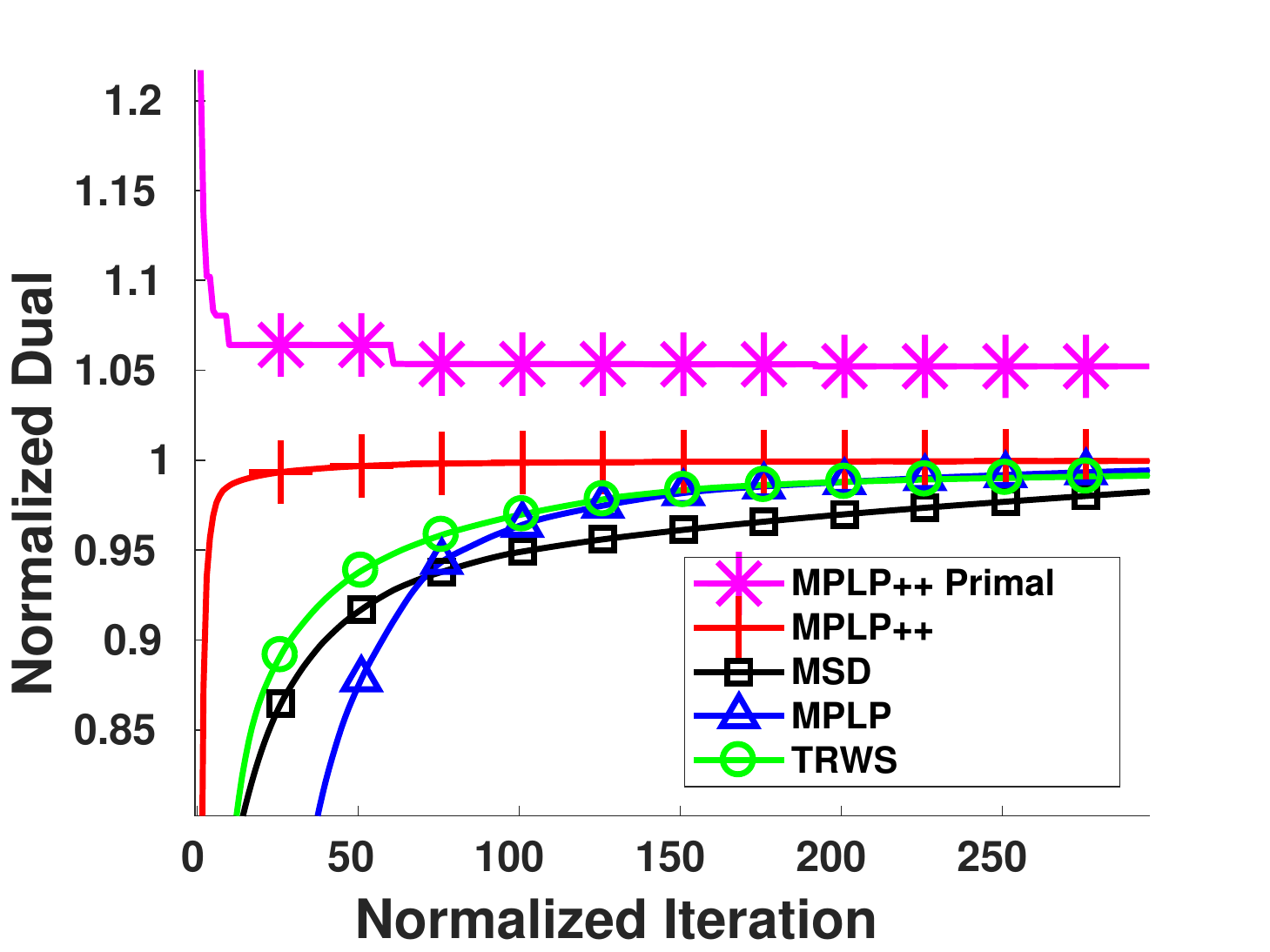}}\label{fig:dual-vs-iters-protein}\end{subfigure}
\begin{subfigure}[{\tt protein}, sec]{\includegraphics[width=0.23\linewidth]{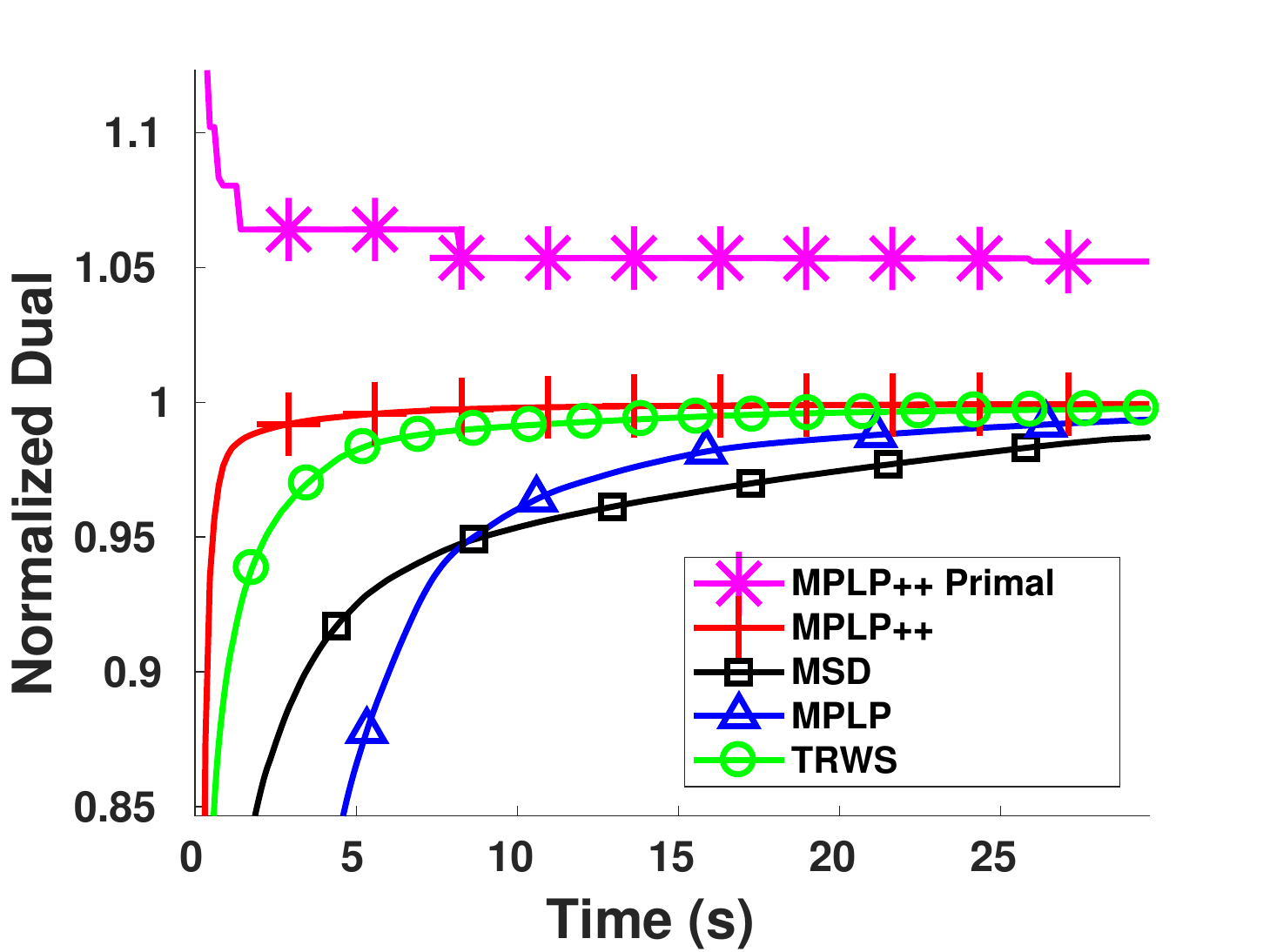}}\end{subfigure}
\begin{subfigure}[{\tt stereo}, iter]{\includegraphics[width=0.23\linewidth]{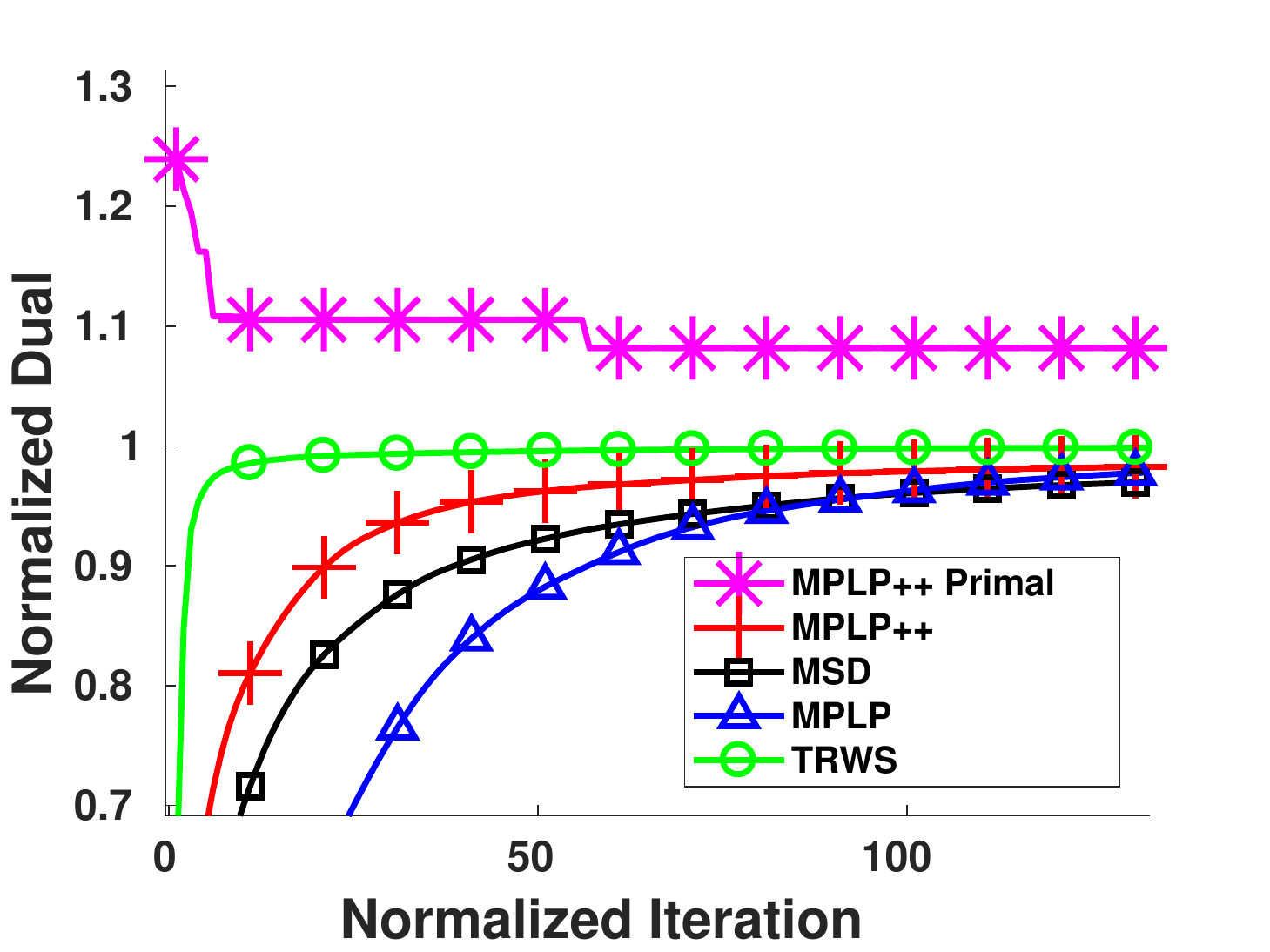}}\label{fig:dual-vs-iters-stereo}\end{subfigure}
\begin{subfigure}[{\tt stereo}, sec]{\includegraphics[width=0.23\linewidth]{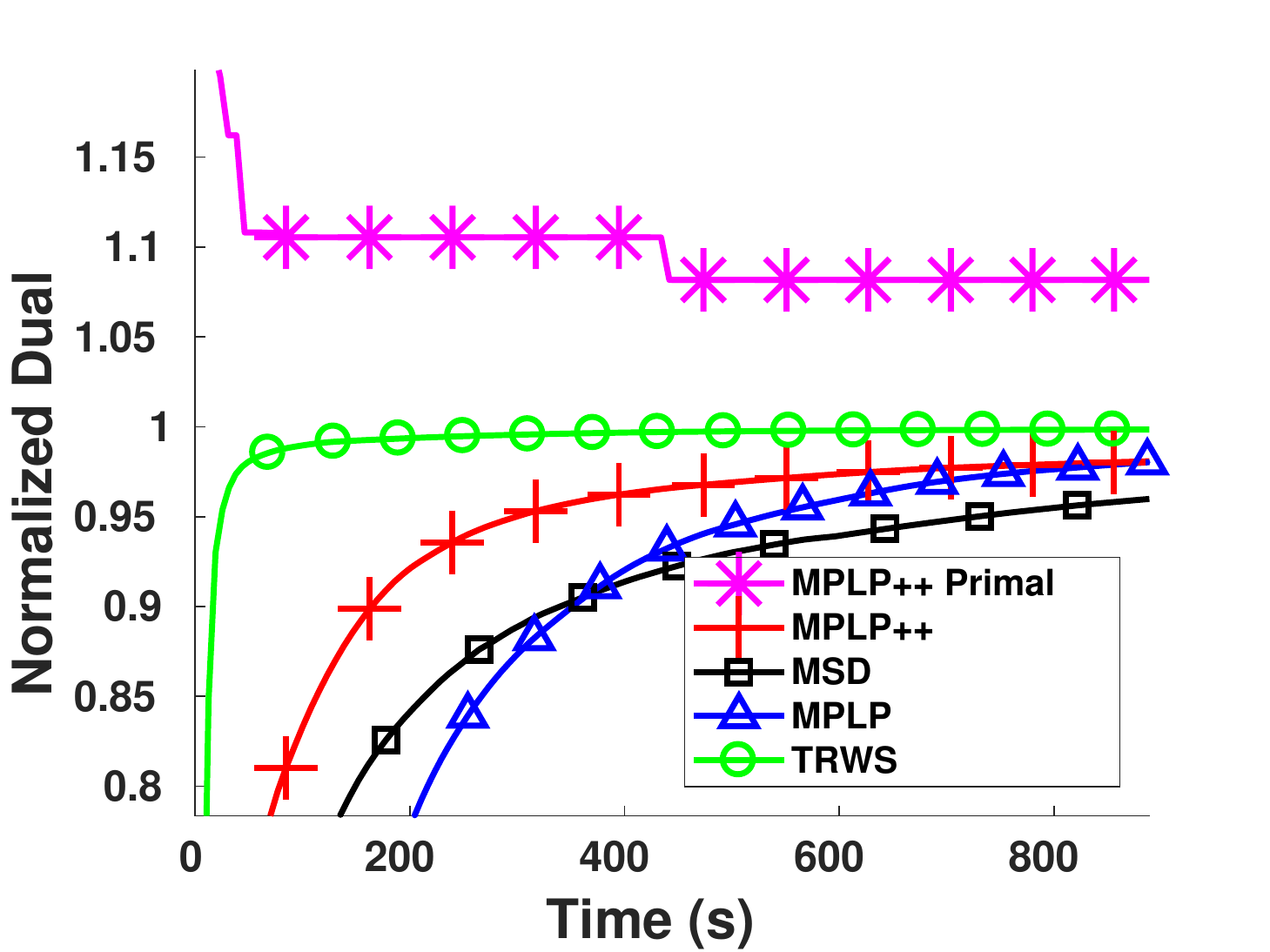}}
\end{subfigure}
\caption{
\label{fig:dual-vs-iters}
Improvement in dual as a function of time and iterations for the protein-folding and stereo dataset. The algorithms we compare have different message-passing schemes and end up doing different amounts of work per iteration. Hence, for a fair comparison across algorithms we define a normalized iteration as exactly $|\mathcal{E}|$ messages. This is also equal to number of messages passed in a normal iteration divided by its oracle complexity. (a) and (b) are for the dense protein-folding dataset, where both per unit time and per iteration,  {\tt MPLP++} outperforms {\tt TRWS} and all other algorithms. In the sparse stereo dataset (c,d) {\tt TRWS} beats all other algorithms. Results are averaged over the entire dataset. The dual is normalized to 1 for equal weighing of every instance in the dataset.}
\end{figure}
\begin{figure}[!th]
\centering
\begin{subfigure}[{\tt pose}]{\includegraphics[width=0.24\linewidth]{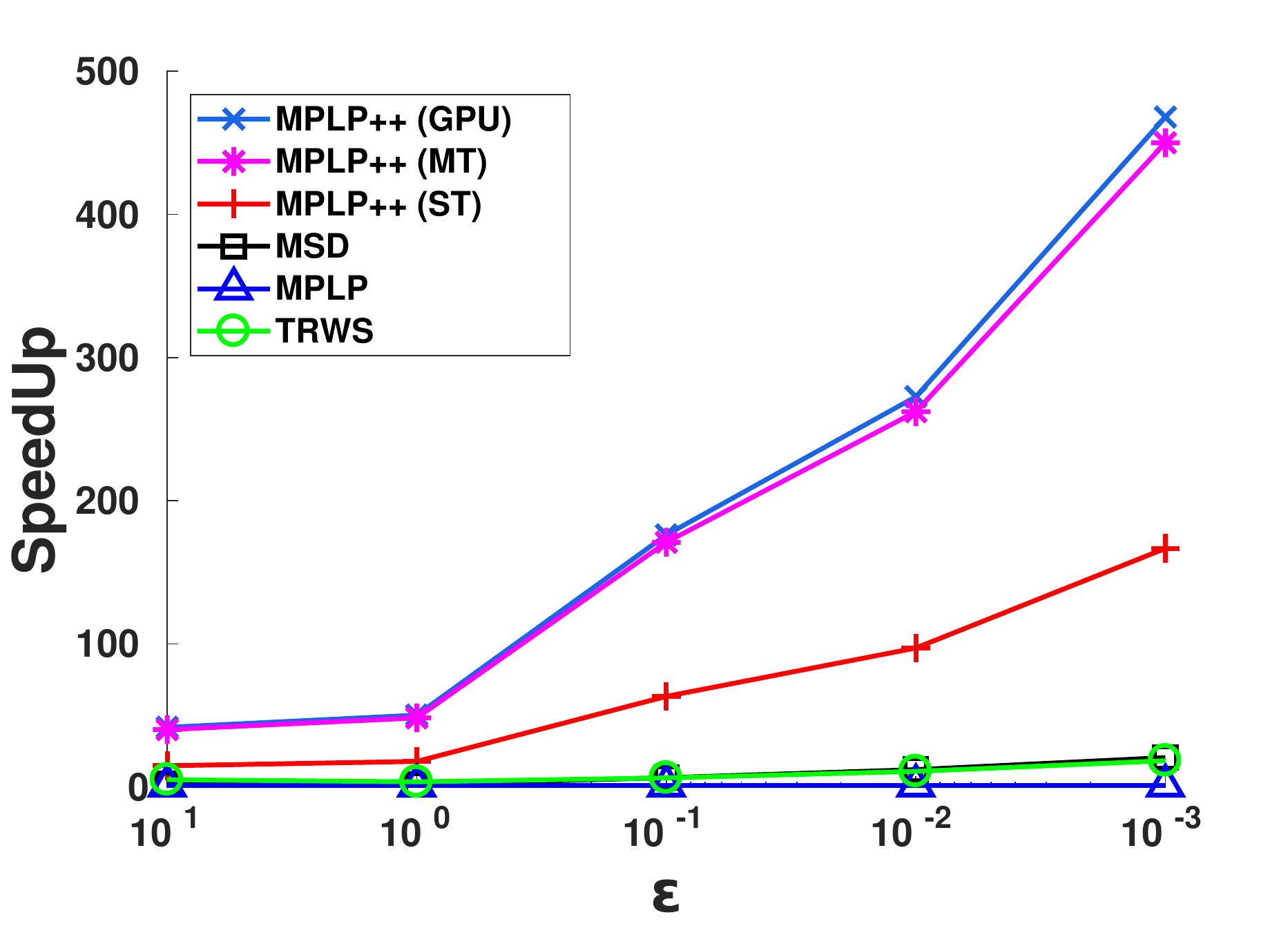}}
\end{subfigure}
\begin{subfigure}[{\tt protein}]
{\includegraphics[width=0.24\linewidth]{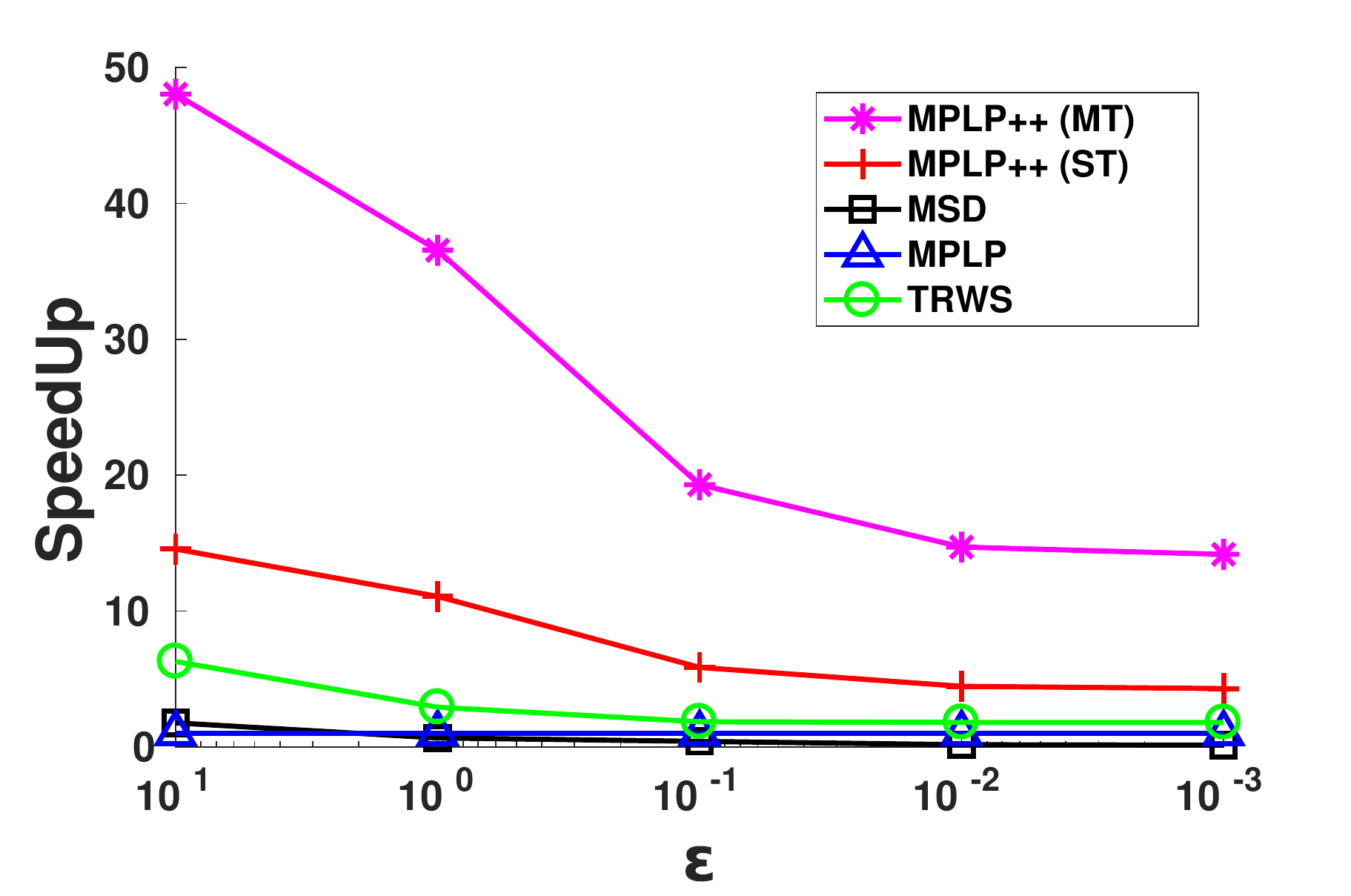}}
\end{subfigure}
\begin{subfigure}[{\tt worms}]
{\includegraphics[width=0.24\linewidth]{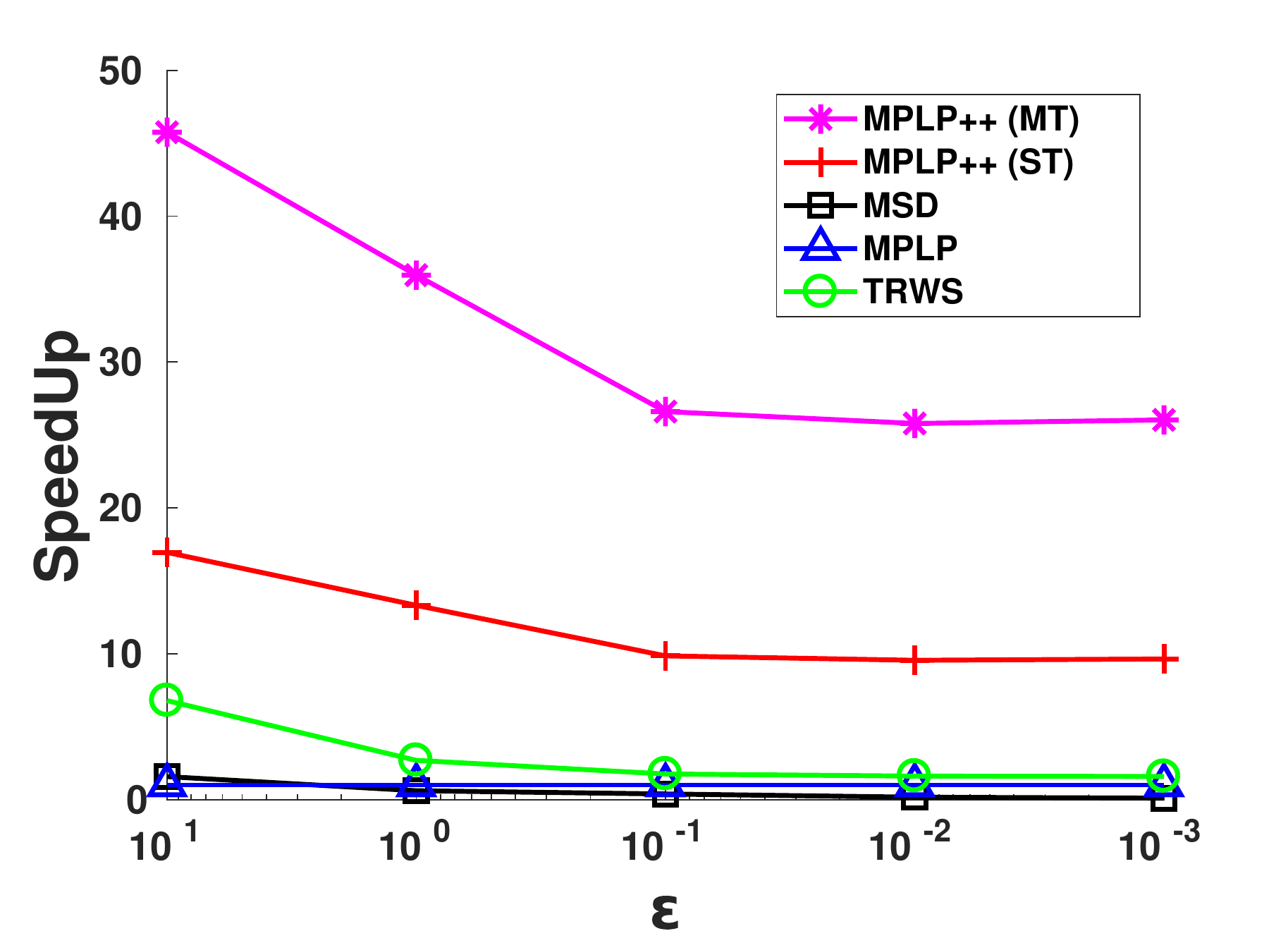}}
\end{subfigure}
\begin{subfigure}[{\tt stereo}]
{\includegraphics[width=0.24\linewidth]{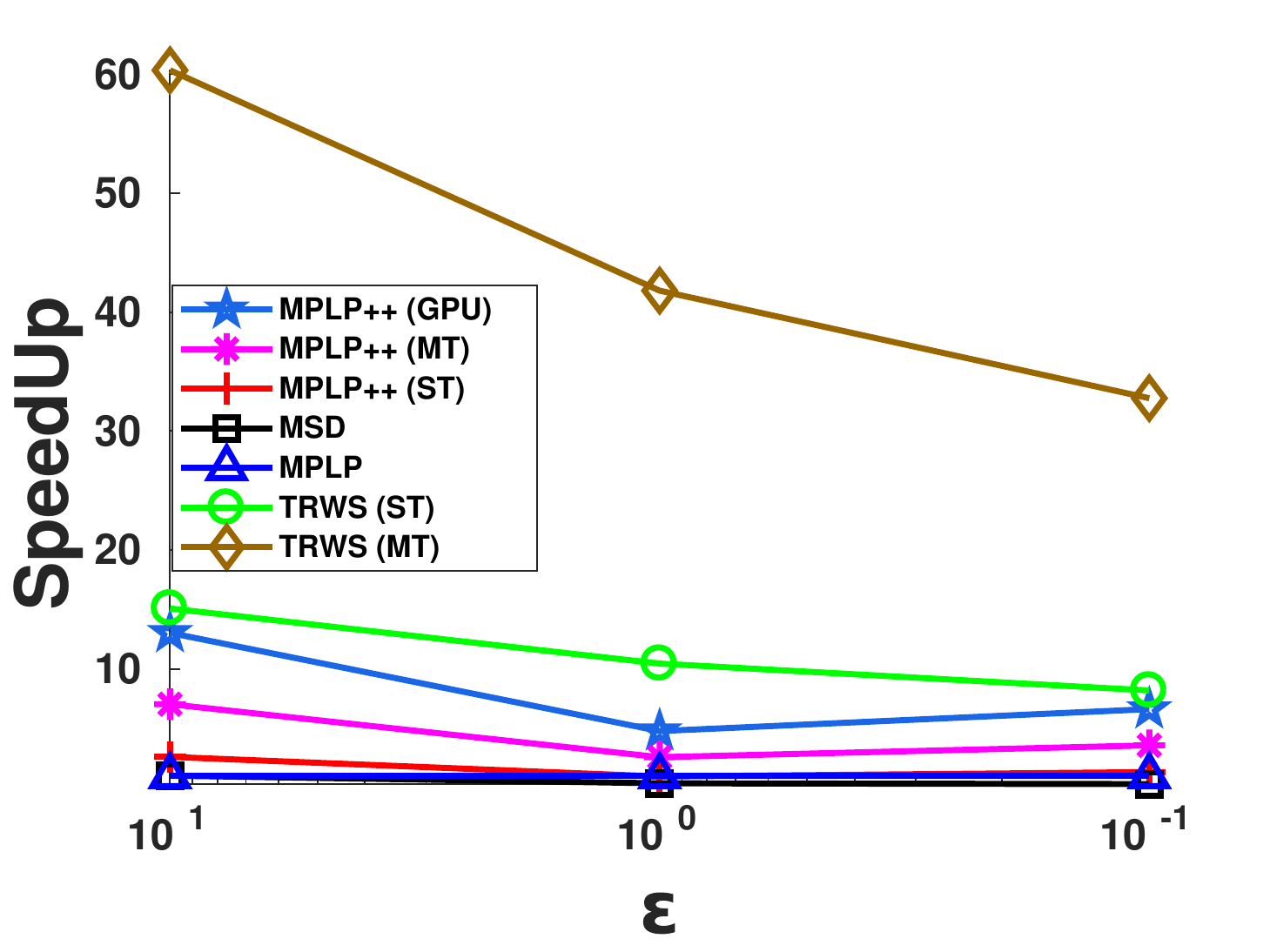}}
\end{subfigure}
\caption{
\label{fig:speedups}
Relative speed w.r.t the {\tt MPLP} algorithm in converging to within $\varepsilon$ of the best attained dual optima $D^{*}(\theta^{\phi})$, i.e. $D^{*}(\theta^{\phi})-\varepsilon$. The plot shows the speedups of all the algorithms relative to MPLP for $\varepsilon$'s $0.001$, $0.01$, $0.1$, $1$ and $10$ of $D^{*}(\theta^{\phi})$. Figure (a) shows {\tt MPLP++} is $50\times$ faster than {\tt MPLP} in converging to within $10\%$ of $D^{*}(\theta^{\phi})$. Figure (b) and (c) likewise show an order of magnitude speedup. Figure (d) shows that for the stereo dataset consisting of sparse graphs {\tt TRWS} dominates {\tt MPLP++}. Convergence only till $0.1\%$ of $D^{*}(\theta^{\phi})$ is shown for stereo as only {\tt TRWS} converges to the required precision.}
\end{figure}
\begin{figure}[!ht]
\centering
\begin{subfigure}[100$\%$]{\includegraphics[width=0.24\linewidth]{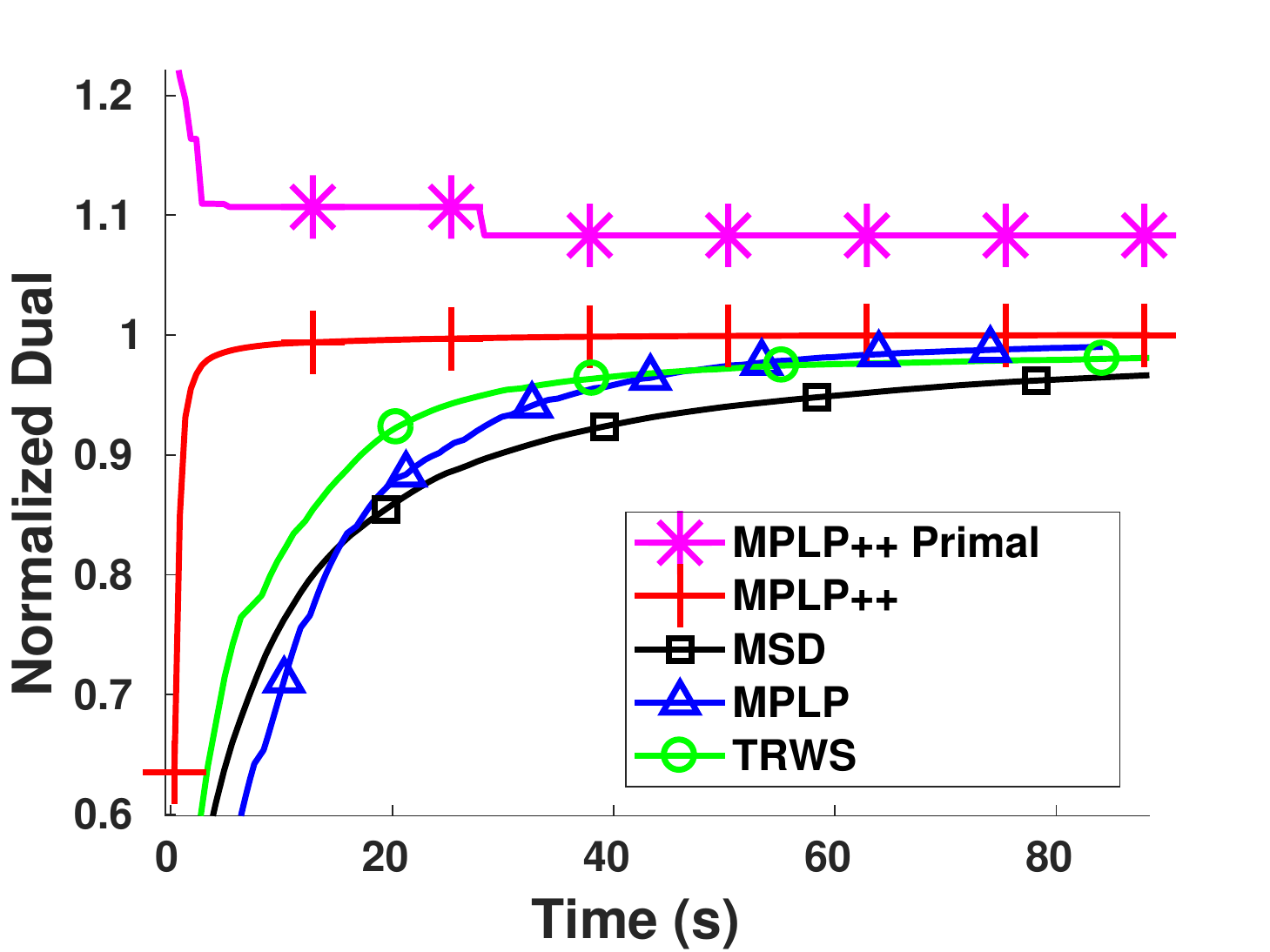}}
\end{subfigure}
\begin{subfigure}[40$\%$]{\includegraphics[width=0.24\linewidth]{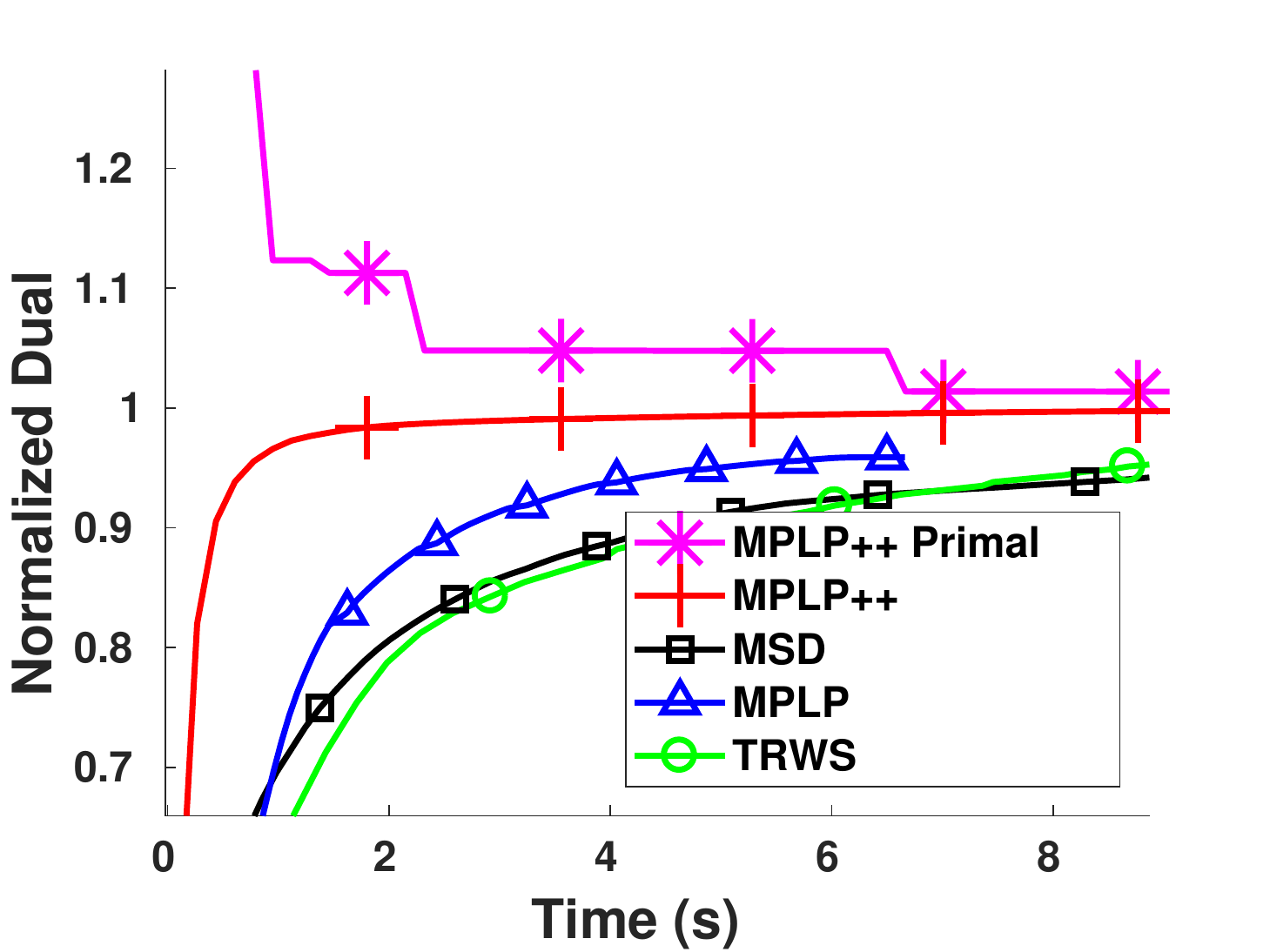}}
\end{subfigure}
\begin{subfigure}[10$\%$]
{\includegraphics[width=0.24\linewidth]{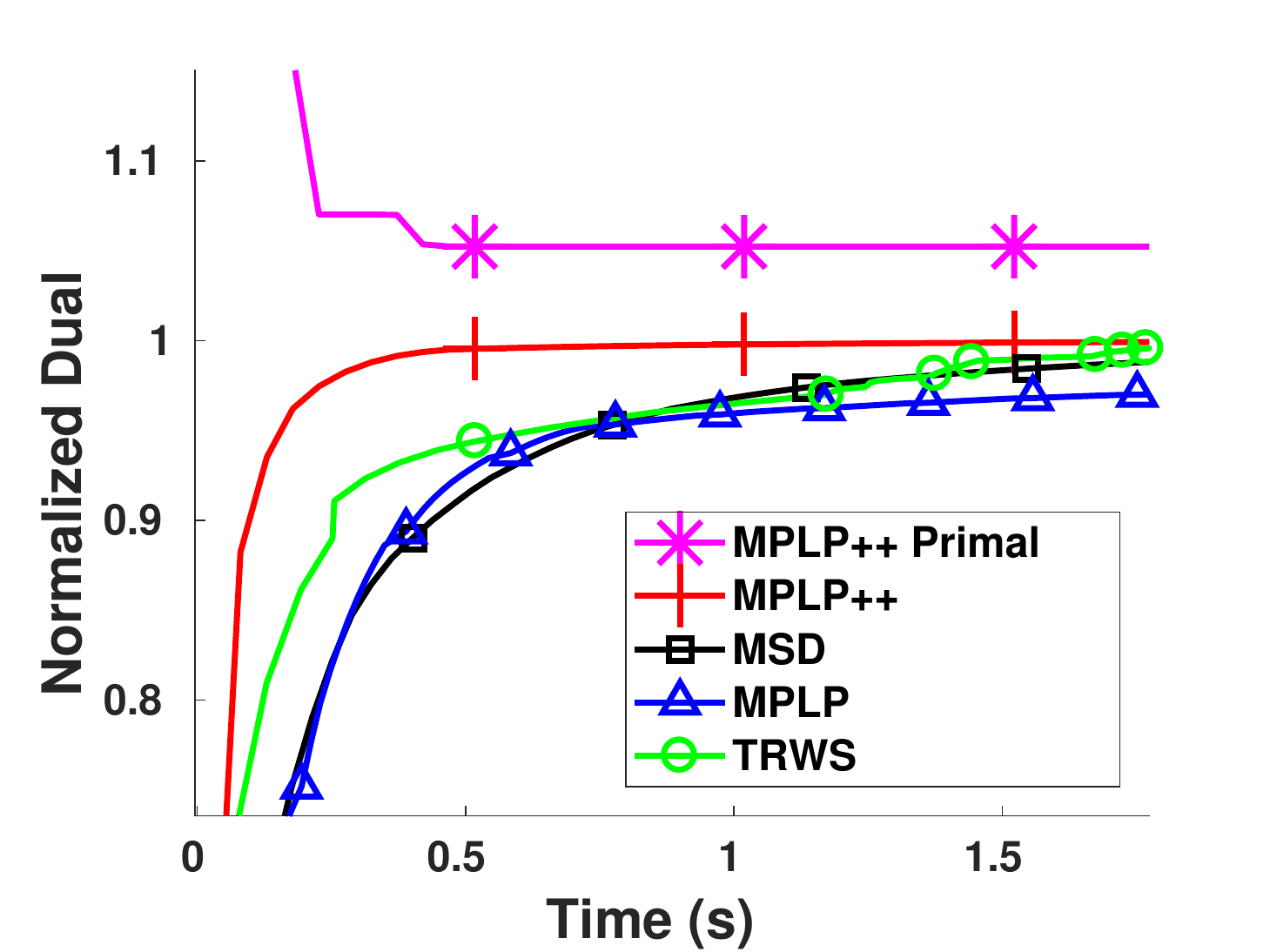}}
\end{subfigure}
\begin{subfigure}[5$\%$]
{\includegraphics[width=0.24\linewidth]{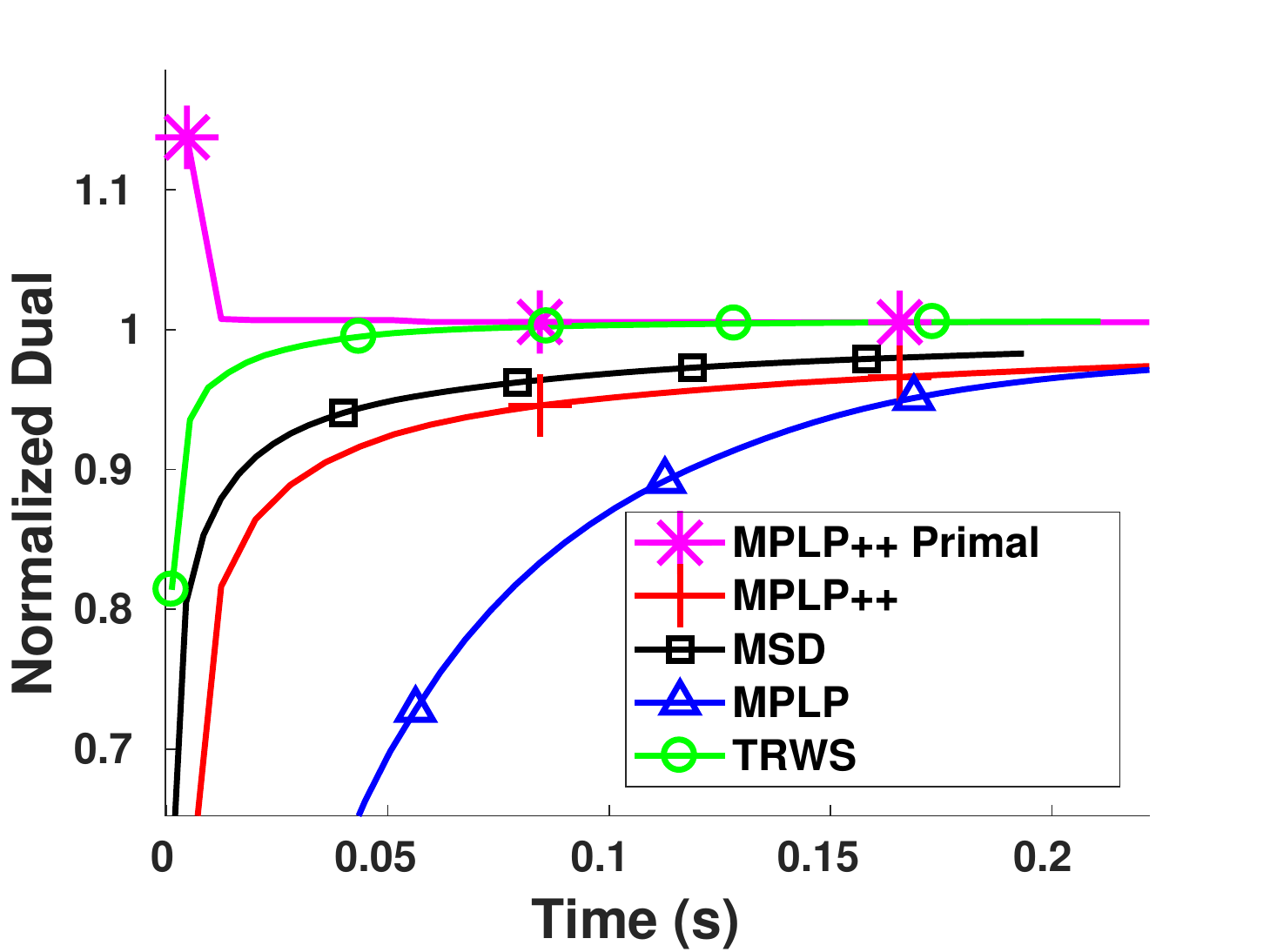}}
\end{subfigure}
\caption{
Degradation with Sparsity (Dual vs Time): (a)-(d) show graphs with decreasing average connectivity given as percentage of possible edges in figure subcaption. In (a)-(c) {\tt MPLP++} outperforms {\tt TRWS}. {\tt MPLP++} is resilient to graph sparsification even when 90$\%$ of the edges have been removed. Only when more than $95\%$ of the edges have been removed as in (d) {\tt TRWS} outperforms {\tt MPLP++}.}
\label{fig:perf-deg-time}
\end{figure}


\section{Conclusions and Outlook}
 Block-coordinate ascent methods remain a perspective research direction for creating efficient parallelizable dual solvers for MAP-inference in graphical models.
 We have presented one such solver beating the state-of-the-art on dense graphical models with arbitrary potentials.
 The method is directly generalizable to higher order models, which we plan to investigate in the future.

{\small
\renewcommand{\bibname}{\protect\leftline{References\vspace{-4ex}}}
\bibliographystyle{splncs}
\bibliography{eccv2018submission}
}

\newpage
\newpage
\appendix
\numberwithin{figure}{section}
\addtocontents{toc}{\protect\setcounter{tocdepth}{2}}
\pagestyle{plain}

\setcounter{figure}{0}
\setcounter{table}{0}
\counterwithin{figure}{section}
\counterwithin{table}{section}
\counterwithin{theorem}{section}
\counterwithin{proposition}{section}
\counterwithin{lemma}{section}

\title{\mytitle \\(ECCV'18 Appendix)}
\author{}
\institute{}
\authorrunning{}
\maketitle

\section{Additional Experimental Results}\label{sec:extra-exp}

\begin{figure*}[!h]
\centering
\begin{subfigure}[Worms]
{\includegraphics[width=0.30\linewidth]{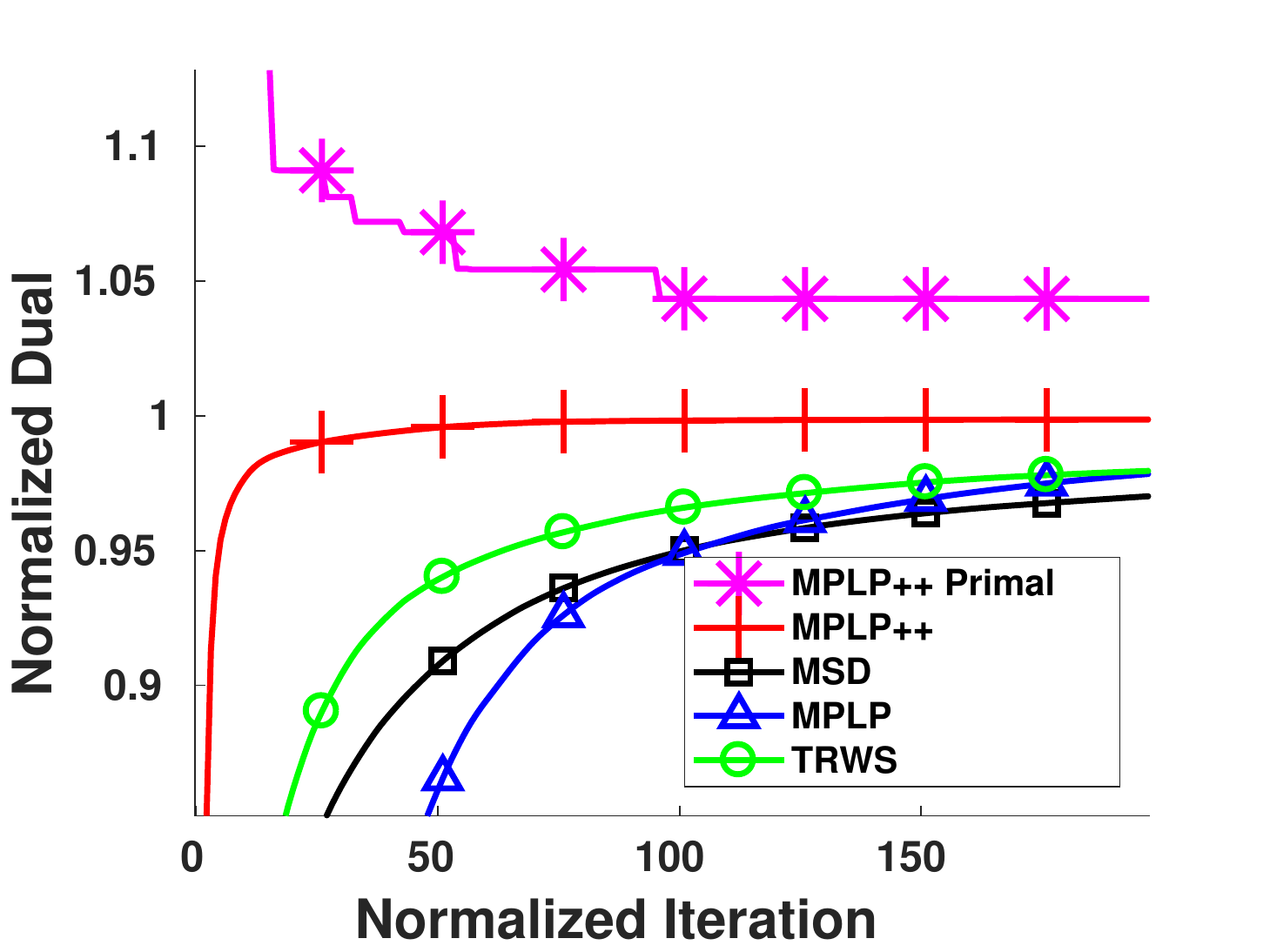}}
\end{subfigure}
\begin{subfigure}[Pose]{
\includegraphics[width=0.30\linewidth]{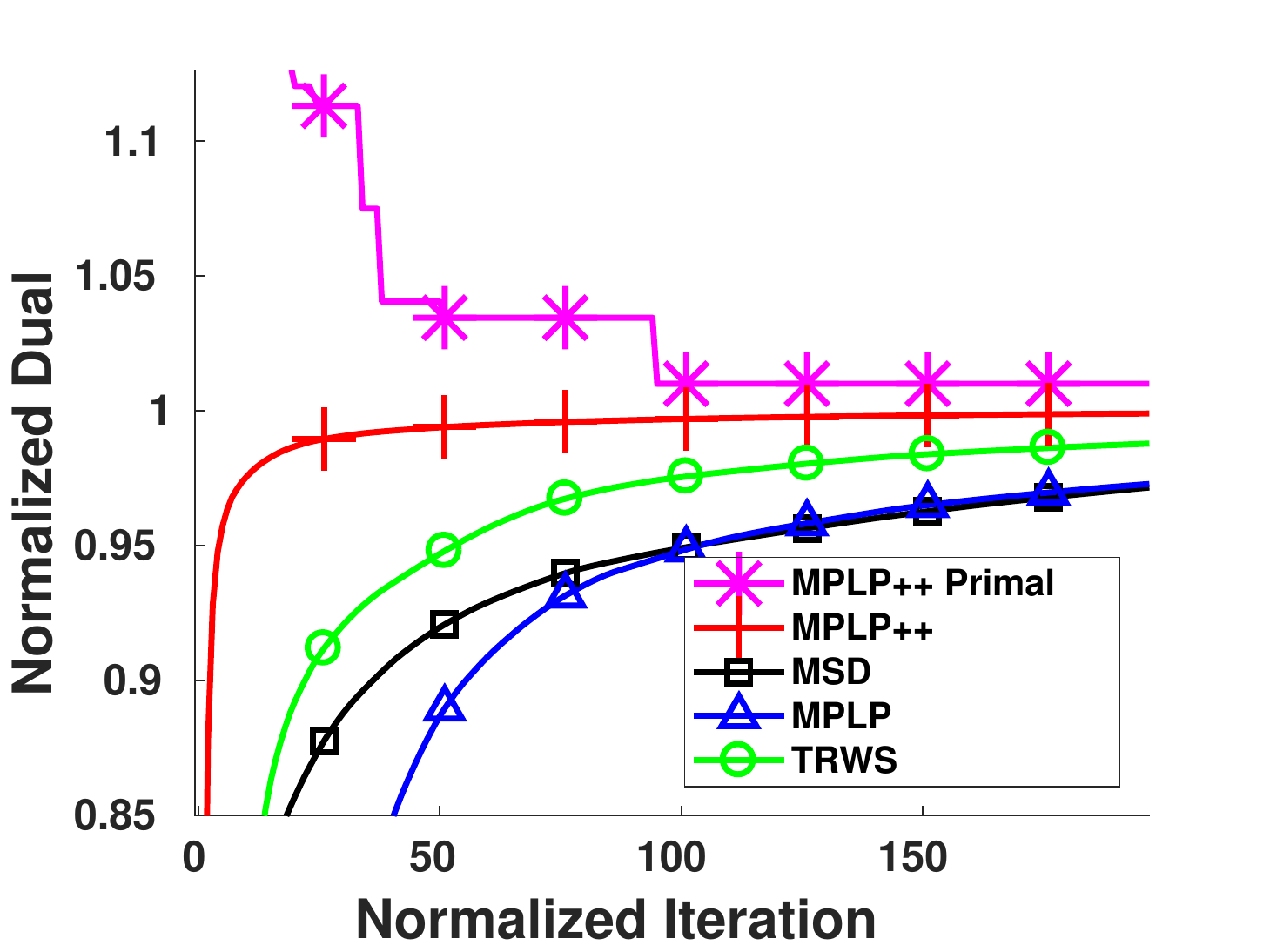}}
\end{subfigure}
\begin{subfigure}[Color-Seg]{\includegraphics[width=0.30\linewidth]{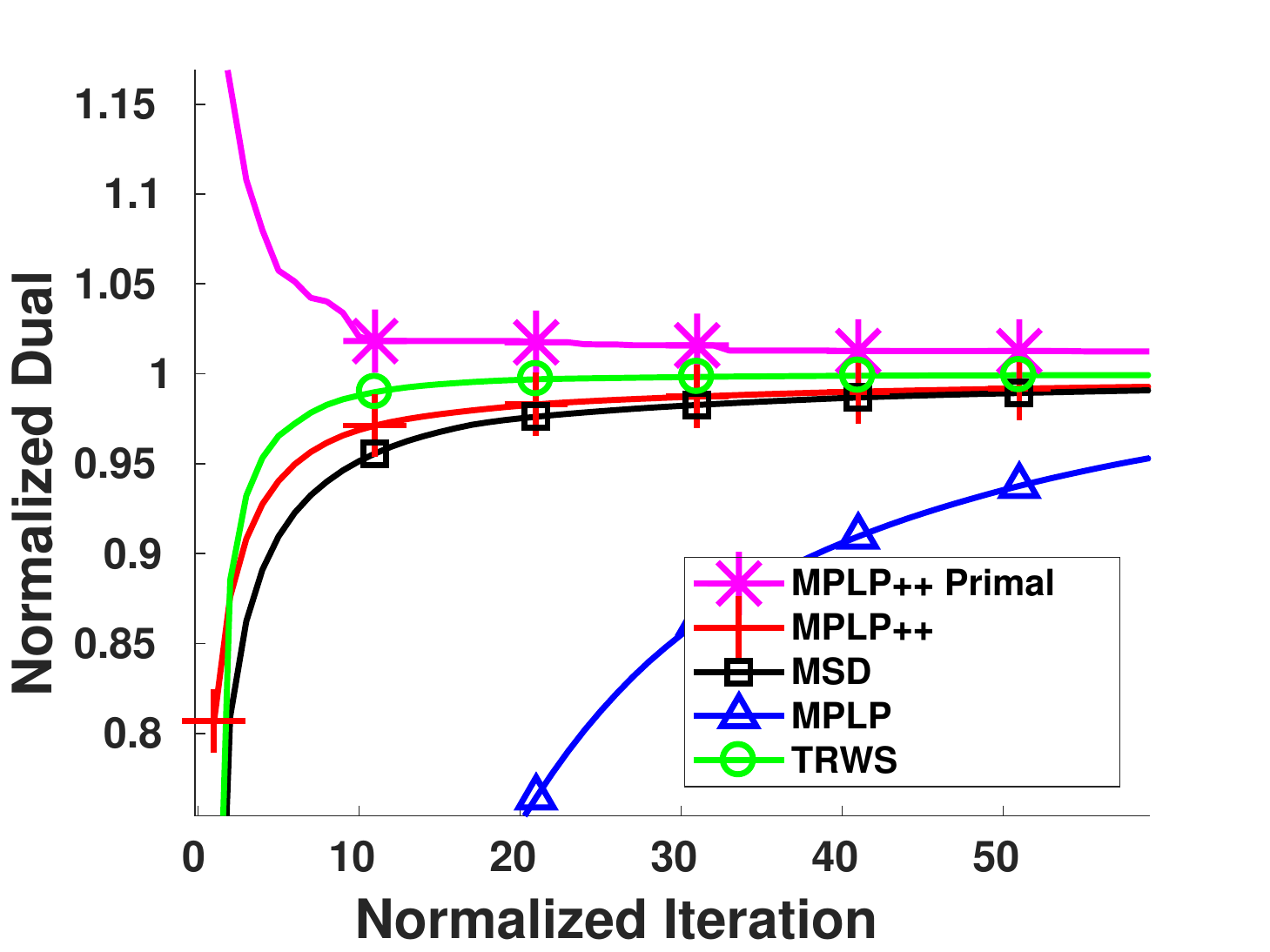}}
\end{subfigure}
\caption{Dual vs Normalized Iterations on three other datasets. The experimental setup and notation is the same as in~\cref{fig:dual-vs-iters}.
Problems in cases (a) and (b) have dense graphs where {\tt MPLP++} outperforms all other algorithms by a substantial margin. In the case (c)  graphs are sparse and {\tt TRWS} is dominant. 
}
\label{fig:dualViters}
\end{figure*}

\begin{figure*}[!h]
\centering
\begin{subfigure}[Worms]{
\includegraphics[width=0.30\linewidth]{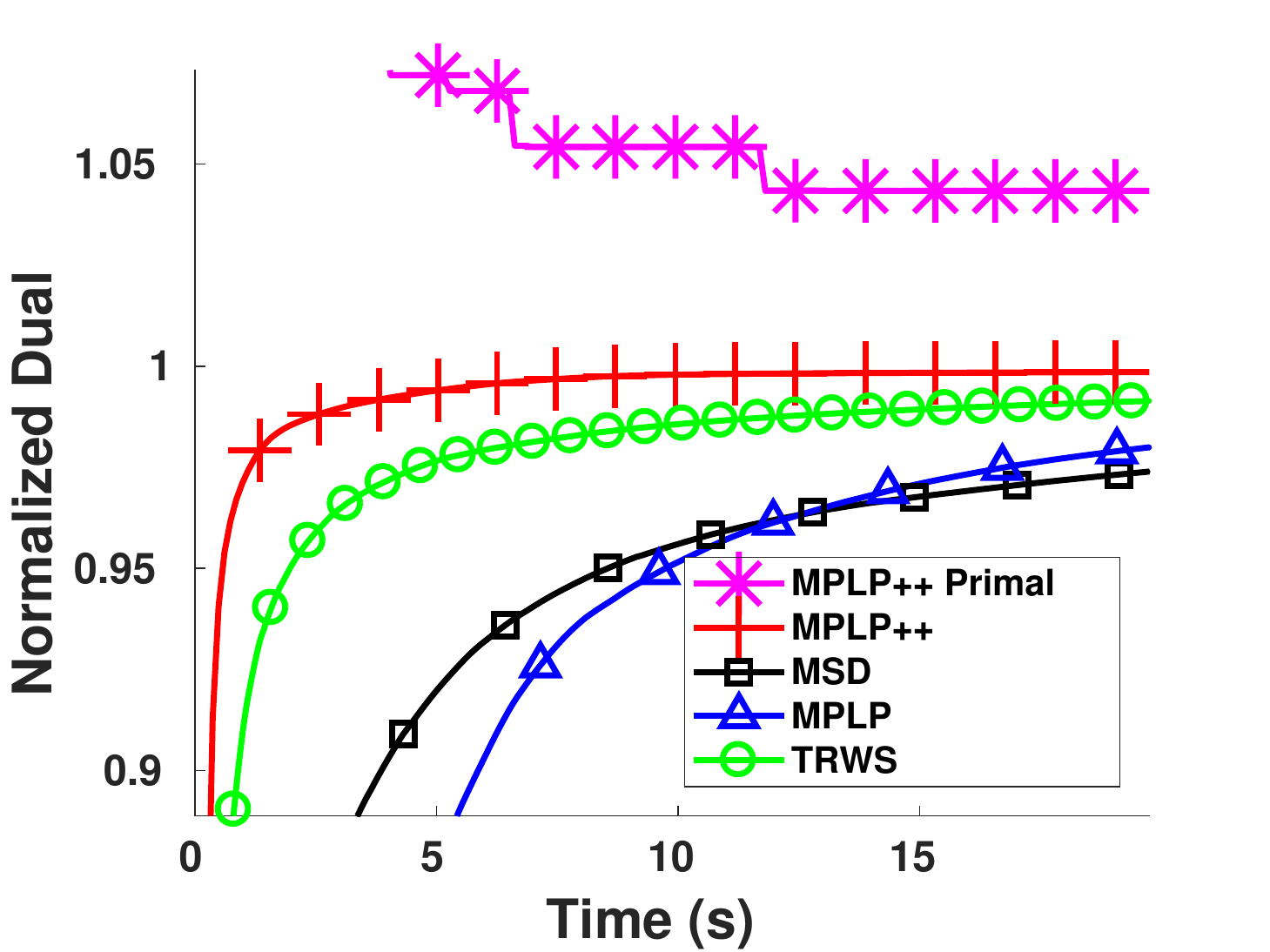}
}
\end{subfigure}
\begin{subfigure}[Pose]{
\includegraphics[width=0.30\linewidth]{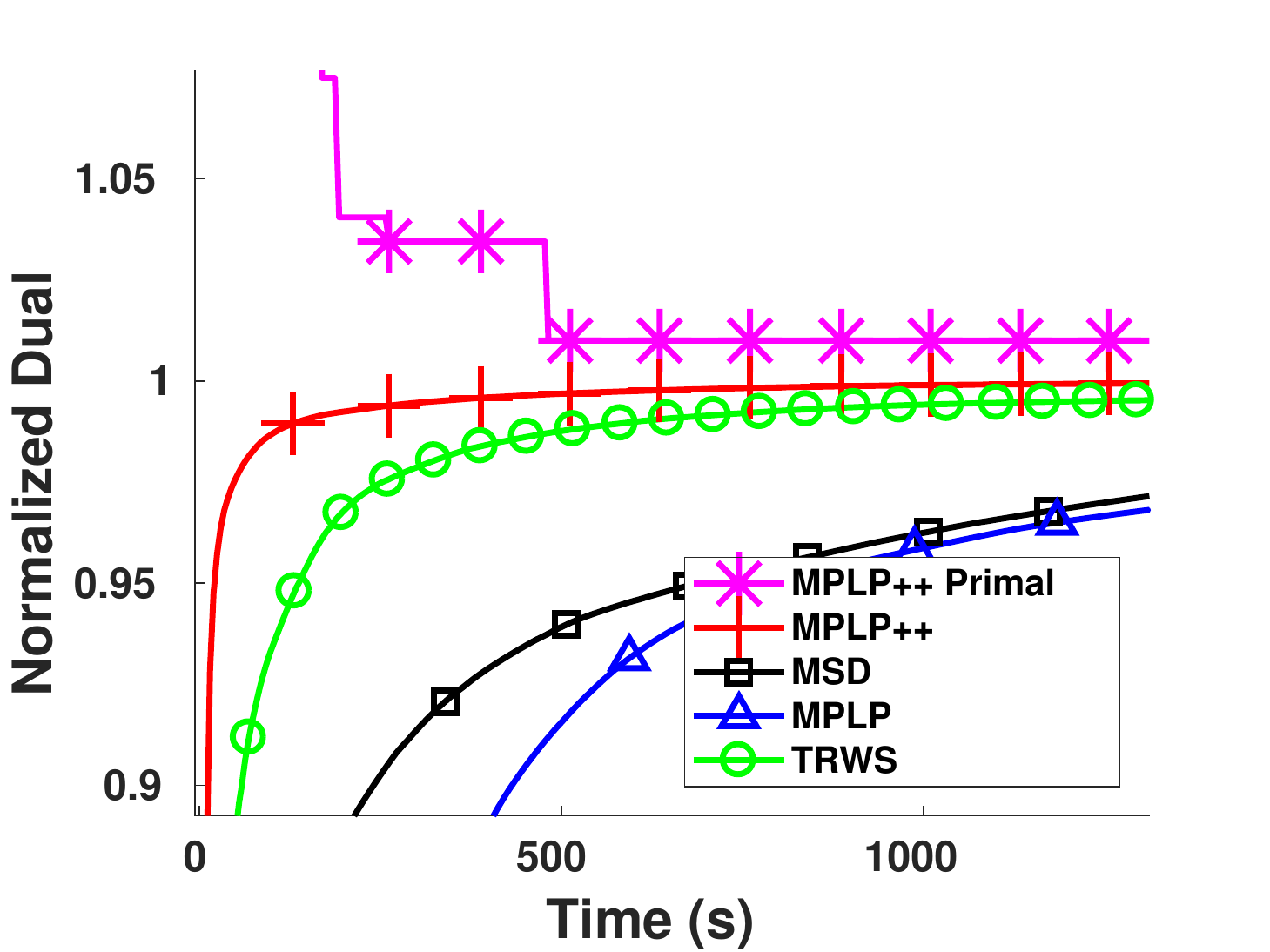}
}
\end{subfigure}
\begin{subfigure}[Color-Seg]{
\includegraphics[width=0.30\linewidth]{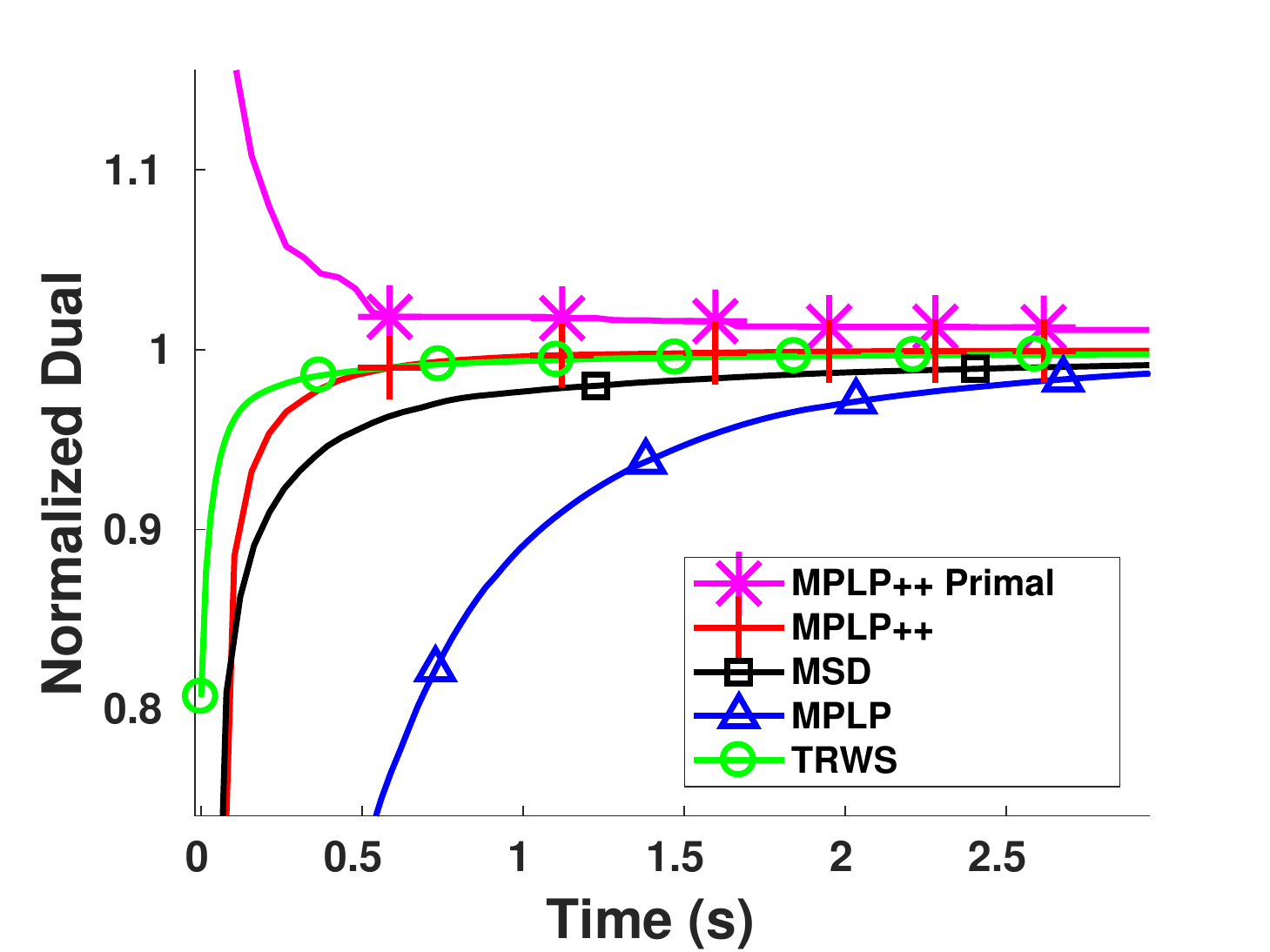}
}
\end{subfigure}
\caption{Dual vs Time (Single Threaded):  Fig. shows dual as a function of time for the single threaded versions of all the algorithms. Following the pattern in ~\cref{fig:dualViters} for dense graphs (a) and (b) {\tt MPLP++} dominates all other algorithms by a considerable margin. For sparse graphs, (c) {\tt TRWS} is the fastest. Both the dual $D(\phi)$ and time have been averaged over the entire dataset and normalized to 1.The curves have been normalized such that each instance of the dataset is weighed equally.}
\label{fig:dualVtime}
\end{figure*}

%

\begin{figure*}[!h]
\centering
\begin{subfigure}[100$\%$]{\includegraphics[width=0.24\linewidth]{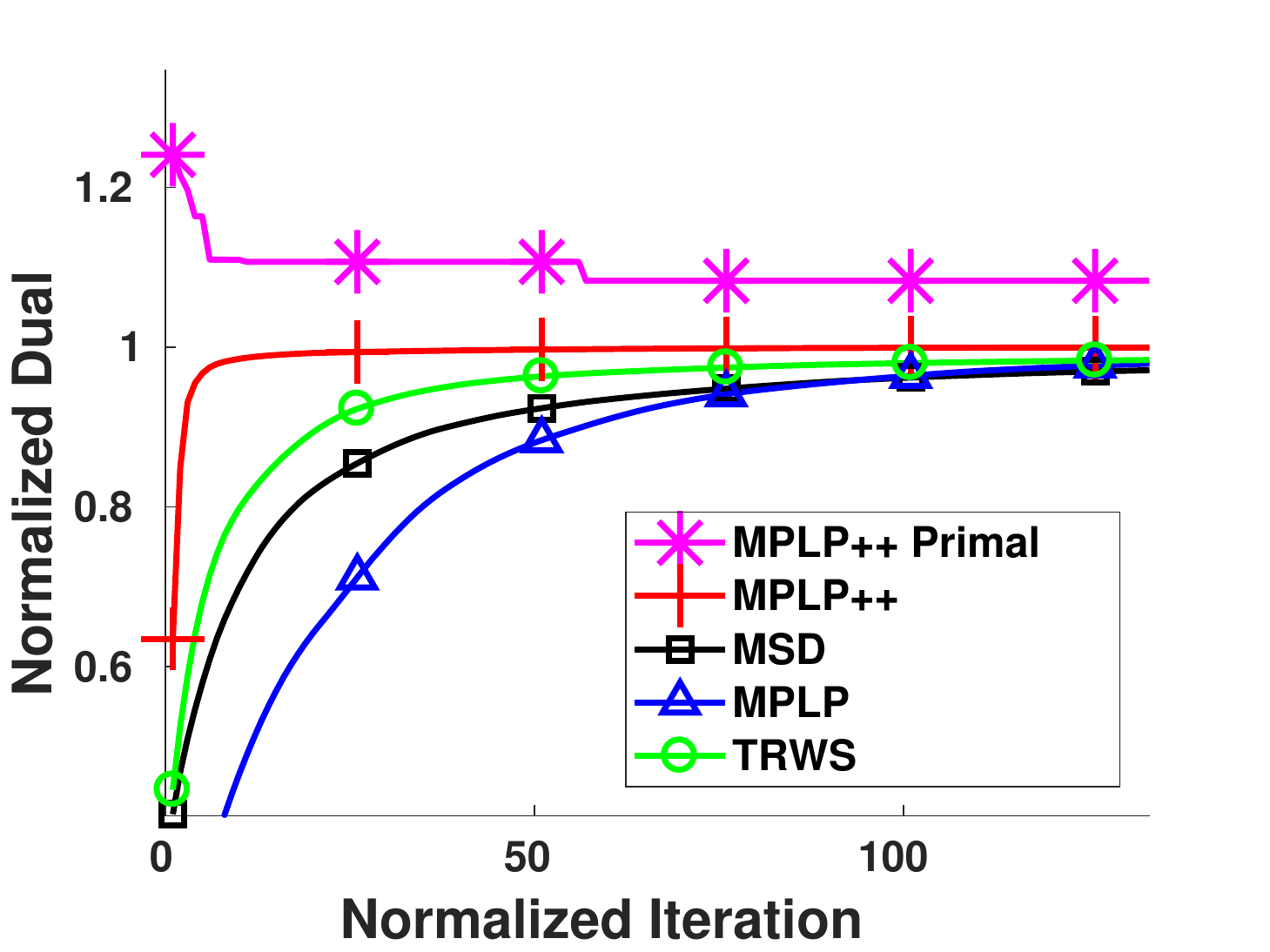}}
\end{subfigure}
\begin{subfigure}[80$\%$]{\includegraphics[width=0.24\linewidth]{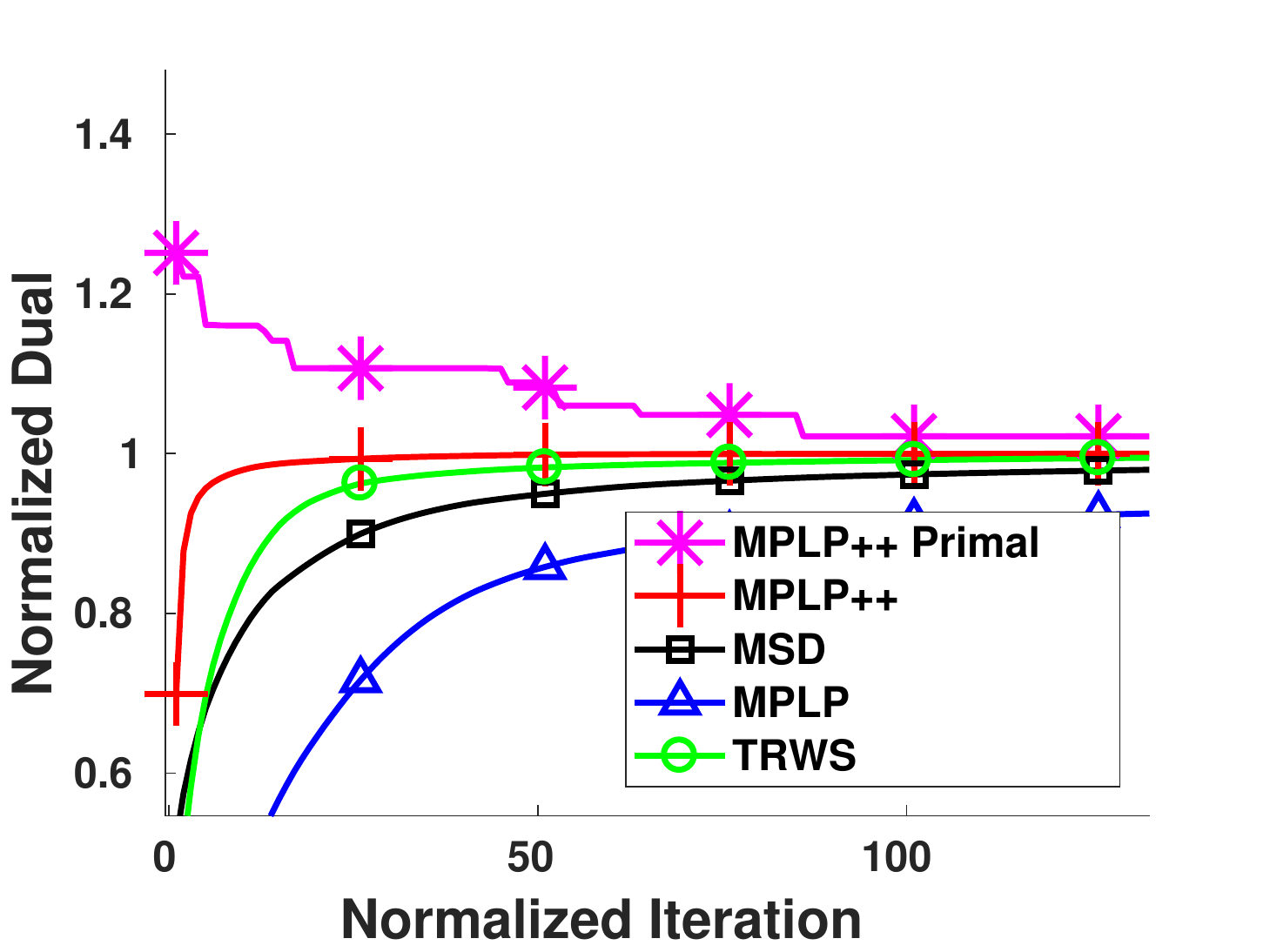}}
\end{subfigure}
\begin{subfigure}[60$\%$]{\includegraphics[width=0.24\linewidth]{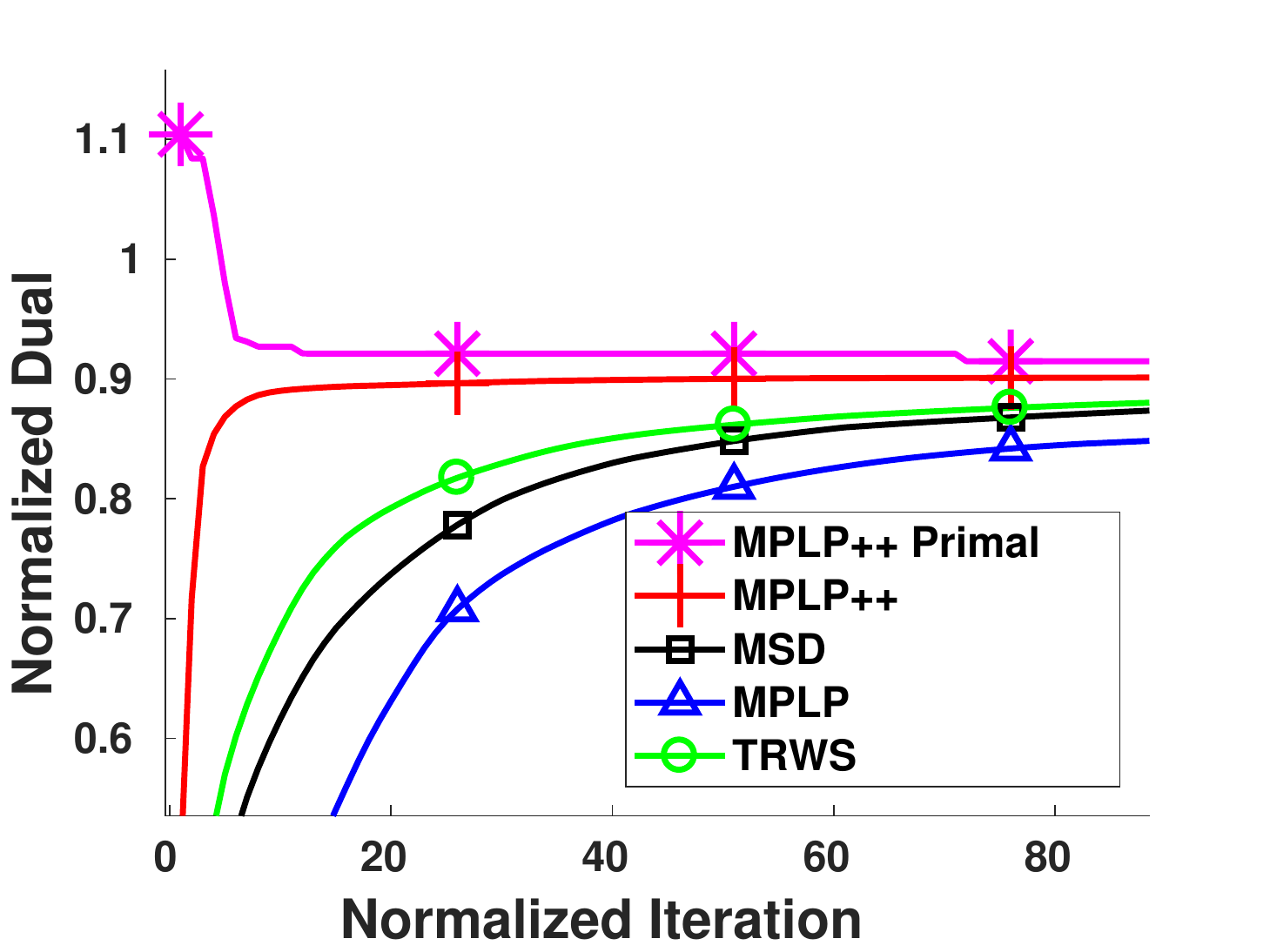}}
\end{subfigure}
\begin{subfigure}[40$\%$]{\includegraphics[width=0.24\linewidth]{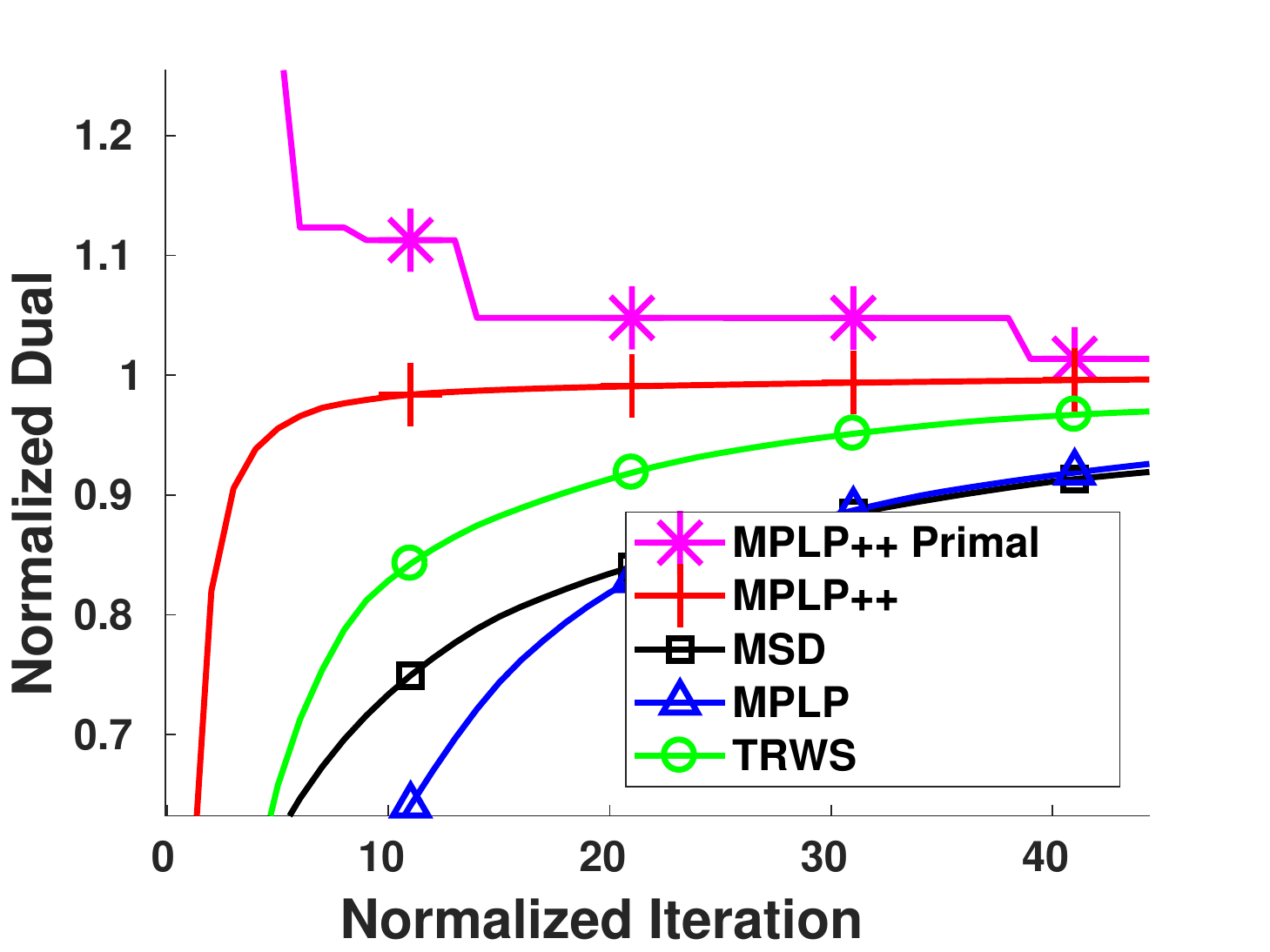}}
\end{subfigure}

\begin{subfigure}[20$\%$]
{\includegraphics[width=0.24\linewidth]{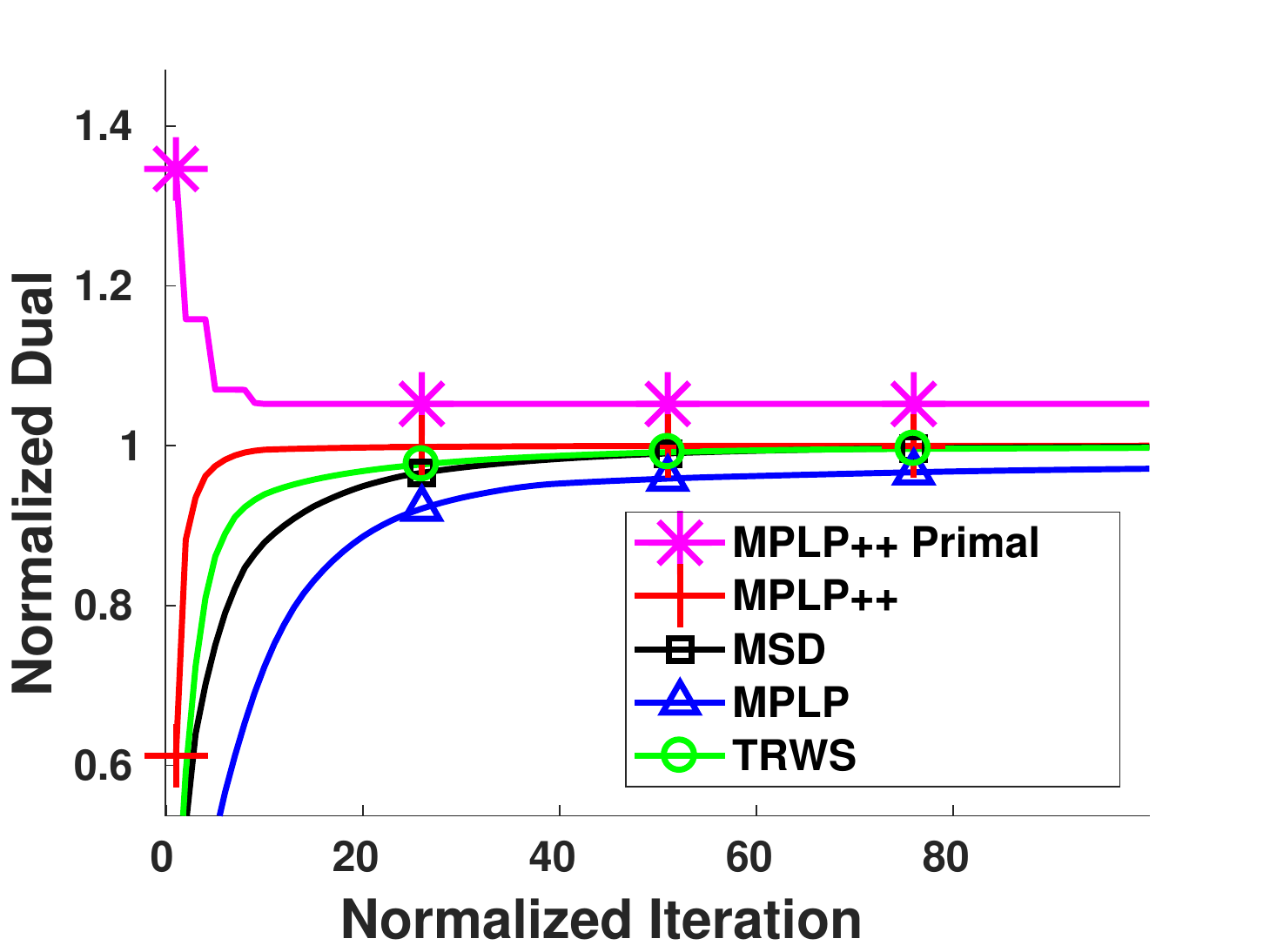}}
\end{subfigure}
\begin{subfigure}[10$\%$]
{\includegraphics[width=0.24\linewidth]{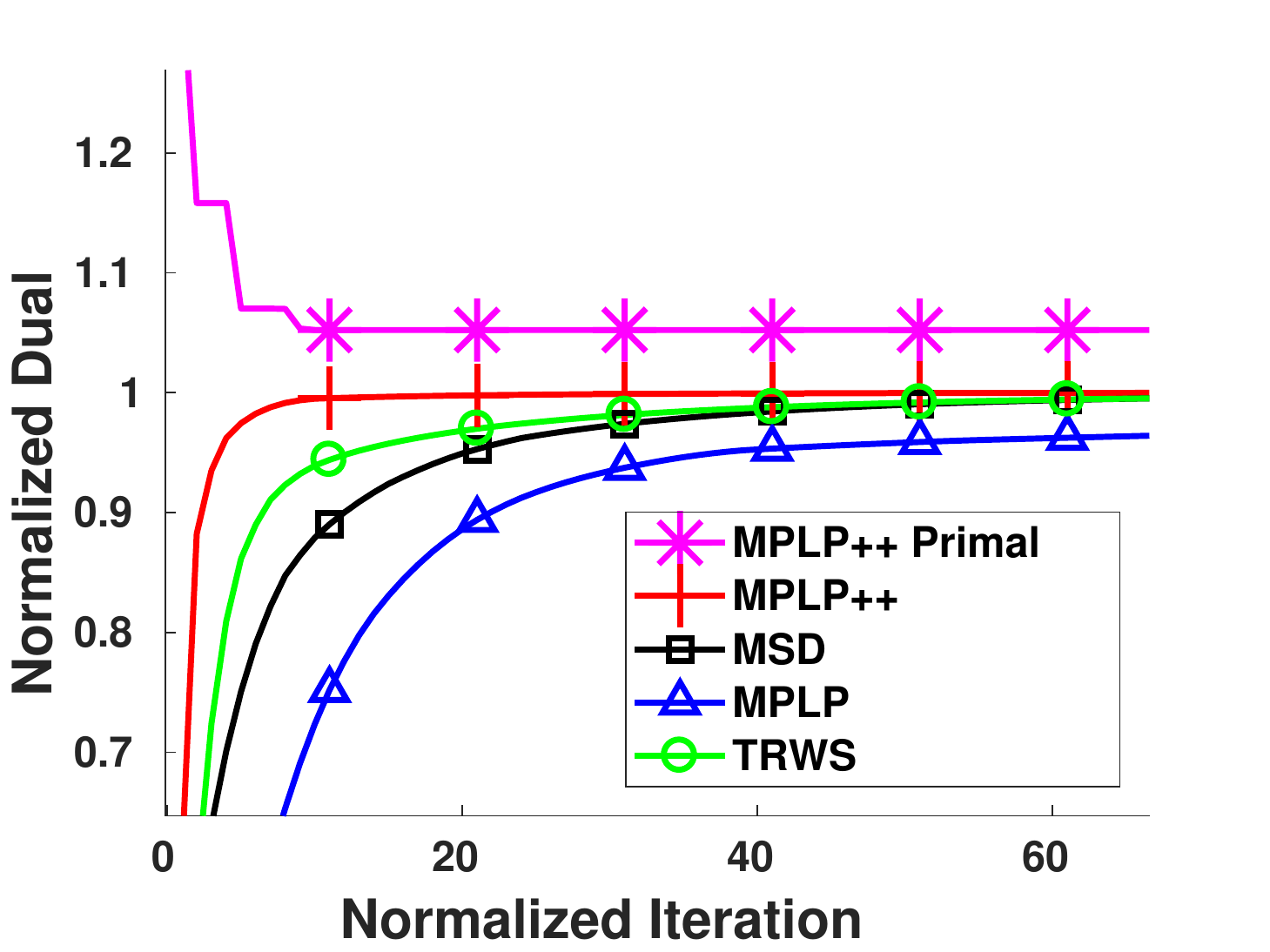}}
\end{subfigure}
\begin{subfigure}[5$\%$]
{\includegraphics[width=0.24\linewidth]{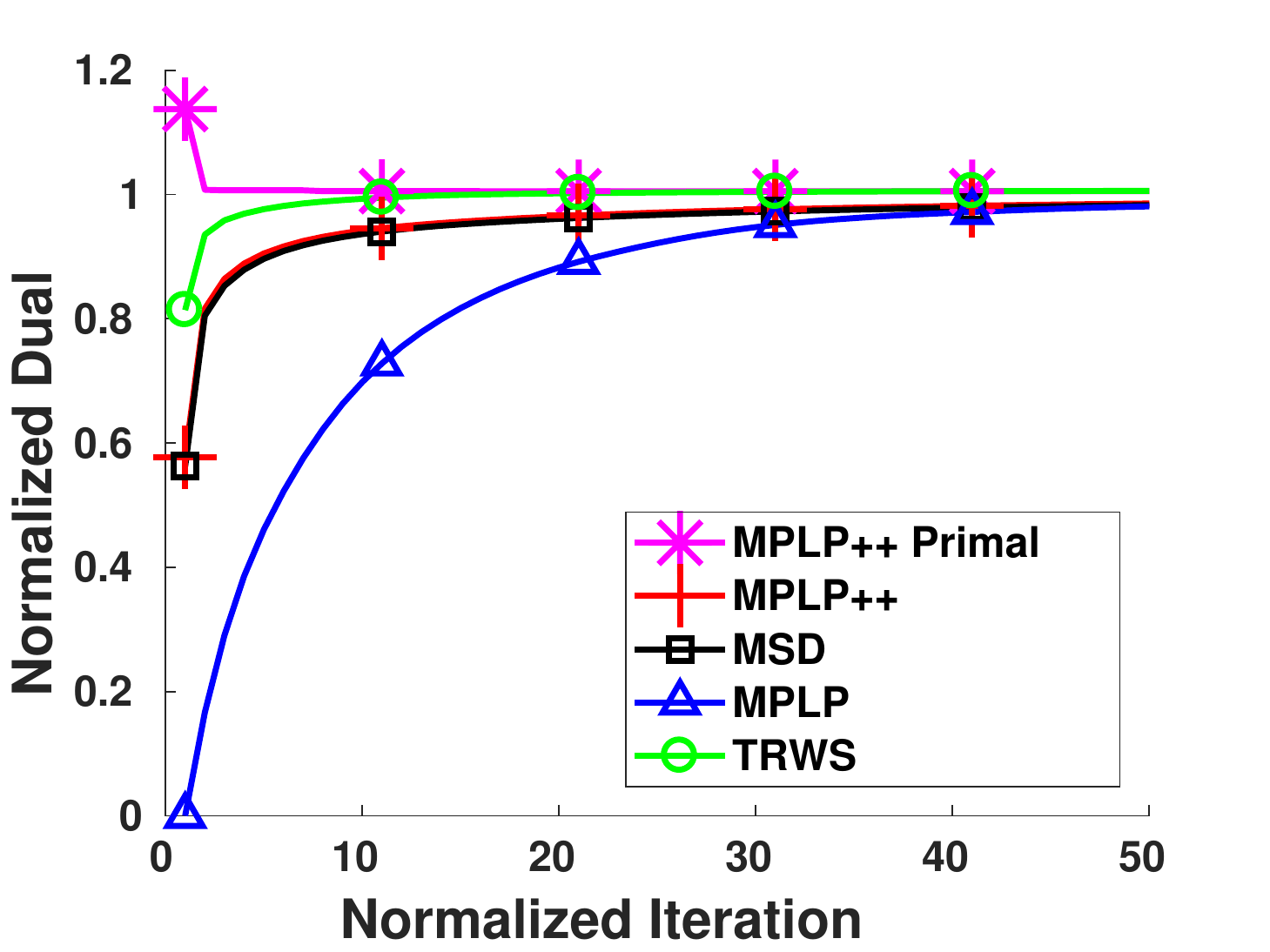}}
\end{subfigure}
\begin{subfigure}[1$\%$]
{\includegraphics[width=0.24\linewidth]{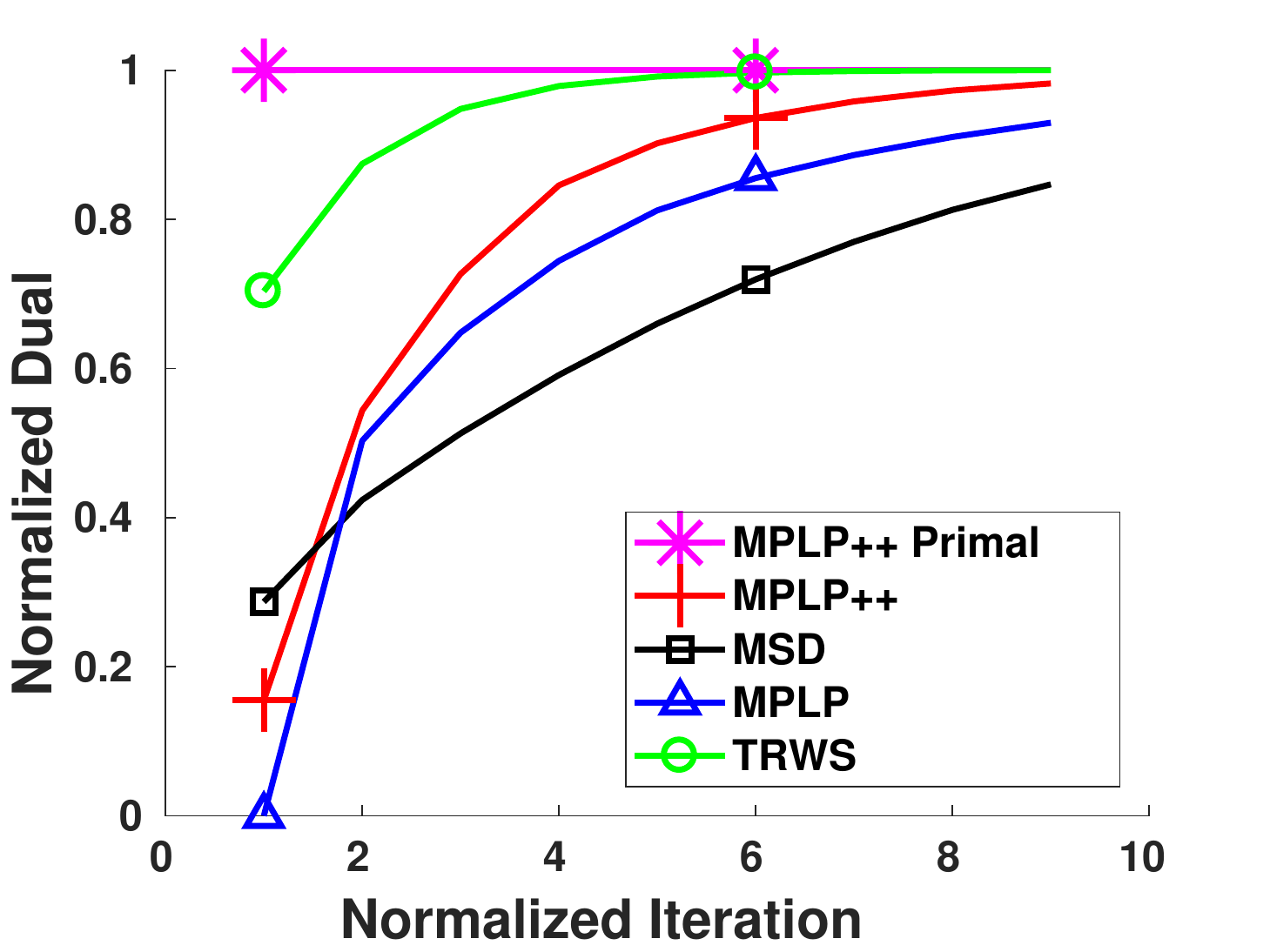}}
\end{subfigure}
\caption{Degradation With Sparsity (Dual vs Iterations): (a)-(h) show graphs with decreasing average connectivity given as percentage of possible edges in figure subcaption. In (a)-(f) {\tt MPLP++} outperforms {\tt TRWS}. Handshake is resilient to graph sparsification even when 90$\%$ of the edges have been removed. For (g) and(h) {\tt TRWS} outperforms handshake.}
\label{fig:perfDegIter}
\end{figure*}

\section{Formal Proofs}

\label{sec:proofs}

\subsection{Additional Notation} \label{sec:add-notation}
To reduce clutter in the proofs we define relevant short-forms here. The set $[n]$ is the set of the first $n$ natural numbers, \ie $[n]=\{1,\hdots,n\}$.  $[m][n]$ denotes the Cartesian product $[m] \times[n]$. We denote the same reparametrized unaries of two different algorithms $\mathcal{A}$ and $\mathcal{B}$ as $\ThA{u}{A}$  and $\ThA{u}{B}$. $\theta \in \mathbb{R}^{\mathcal{I}}$ is the vector stacked up of unary $\theta_u$ and pairwise $\theta_{uv}$ potentials. Let $\mathcal{O}_{u}^{\varepsilon}(\theta) =\{s \mid \theta_u(s) \leq \min_{s'}
\ \theta_u(s') + \varepsilon \}$ be the set of labellings within $\varepsilon>0$ of the optimal $\theta_u$. $\mathcal{O}_{uv}^{\varepsilon}(\theta)$ is similarly defined for $\theta_{uv}$. Let also $\mathcal{O}^{\varepsilon}(\theta)=\{\mathcal{O}^{\varepsilon}_u(\theta) \mid  \ \forall u \in \SV\} \cup \{\mathcal{O}^{\varepsilon}_{uv}(\theta) \mid  \ \forall uv \in \SE\}$. 

\begin{definition}
Tolerance factor $\varepsilon$ is the minimum value for which $\mathcal{O}^{\varepsilon}$ contains a consistent labelling, \ie one with node-edge agreement. 
\end{definition}

$\varepsilon$ is a function of $\theta$ and will be written sometimes as $\varepsilon(\theta)$. $\mathcal{O}^{\varepsilon}(\theta)$ can also be thought of as the subset of unary and pairwise labels that are within $\varepsilon$ of the optimal labelling.

$\mathcal{O}^{0}_{u}(\theta)=\{s \mid \theta_u(s)= \min_{s' \in \SY}
\theta_u(s')$\}, then represents the set of optimal labellings. $\mathcal{O}^{0}_{uv}(\theta)$ can be similarly defined for pairwise potentials $\theta_{uv}$. Then, $\mathcal{O}^{0}(\theta)=\{\mathcal{O}^{0}_u(\theta) \mid u \in \SV \} \cup \{\mathcal{O}^{0}_{uv}(\theta) \mid uv \in \SE \}$.

The {\tt MPLP++} operator $\mathcal{H}$ can act both on $g_{uv}$ like pairwise costs and on the entire set of costs $\theta \in \mathbb{R}^{\SI}$. In the former case, it is exactly as defined in Eqn.~\cref{equ:handshake-update}. The latter case corresponds to an iteration of $\mathcal{H}$. The $i$-times composition operation of $\mathcal{H}$ on $\theta$ denotes $i$ iterations of $\mathcal{H}$ on $\theta$,  \ie $\underbrace{\mathcal{H}(\mathcal{H}(...(\mathcal{H}(\theta))))}_{\text{i times}}=\mathcal{H}^{i}(\theta)$. Which is the case would be clear from the argument the $\mathcal{H}$ operator takes. Also, to denote the resulting cost vector after $i$ iterations of $\mathcal{H}$ on $\theta$, we use notation $\theta^{i}=\mathcal{H}^{i}(\theta)=\mathcal{H}(\theta^{i-1})$.

Let $\Theta_{un}(\phi,\kappa)$ be a function that measures the max absolute difference between unary reparameterizations $\theta^{\phi}$ and $\theta^{\kappa}$, defined by $\Theta_{un}(\phi,\kappa)= \max_{u \in \SV, s \in \SY}| \theta^{\phi}_{u}(s) - \theta^{\kappa}_{v}(s)|$. 
Likewise for pairwise costs, $\Theta_{pw}(\phi,\kappa)=\max_{uv \in \SE, st \in \SY^2} | \theta^{\phi}_{uv}(s,t) - \theta^{\kappa}_{uv}(s,t)|$.

We also assume that both unary and pairwise costs have been normalized, \ie $\min_{s \in \SY} \theta_u(s)=0$ and $\min_{st\in\SY^{2}}\theta_{uv}(s,t)=0$. A consequence of normalization is $\theta_u \geq 0$ and $\theta_{uv} \geq 0$. The subtracted cost does not affect the labelling and is added to a constant term.

\Pbcadominances*
\begin{proof}
$\mathbfcal{H} \geq \mathbfcal{M}$: Let $g_{uv}=\theta_{u}+\theta_{v}+\theta_{uv}$. Consider the first line of ~\cref{equ:handshake-update}. 

\begin{equation}
\theta^\H_u(s) \textstyle := \theta^\M_u(s),\quad \theta^\H_v(s) \textstyle := \theta^\M_v(s),\ \forall s \in \SY\
\end{equation}

The first line assigns to the reparameterized {\tt MPLP++} unaries, ($\theta^\H_u$, $\theta^\H_v$) the unaries resulting from applying the {\tt MPLP} update to $g_{uv}$. At this point we have $\H=\M$.  
\par\noindent
Now, consider the subsequent equations of $\H$

\begin{equation}
\begin{aligned}
\theta^\H_v(t) & \textstyle := \theta^\H_v(t) + \min_{s \in \Y}[g_{uv}(s,t) - \theta^\H_v(t)- \theta^\H_u(s)],\ \forall t \in \SY\,\\
\notag
\theta^\H_u(s) & \textstyle := \theta^\H_u(s) + \min_{t \in \Y}[g_{uv}(s,t) - \theta^\H_v(t)- \theta^\H_u(s)],\ \forall s \in \SY\,.
\end{aligned}
\end{equation}

To $\theta^\H_v(t)$, a non-negative quantity $\min_{s \in \Y}[g_{uv}(s,t) - \theta^\H_v(t)- \theta^\H_u(s)]$ is added. Similarly, for  $\theta^\H_u$. Thus, $\H\geq\M$.

\textbf{$\mathbfcal{M} \geq \mathbfcal{U}$}:$ \ThA{u}{M}(s) = \textrm{min}_{t'} g_{uv}(s,t') \geq  \textrm{min}_{s',t'} g_{uv}(s',t') =   \ThA{u}{U}(s)$. Likewise, $ \ThA{v}{M}(t) = \textrm{min}_{s'} g_{uv}(s',t) \geq  \textrm{min}_{s',t'} g_{uv}(s',t') =   \ThA{v}{U}(t)$. Thus, $\mathcal{M} \geq \mathcal{U}$. \QED
\end{proof}

\Pmonotone*
\begin{proof}
Let	 $\theta_u \geq \theta'_u$ and $\theta_v \geq \theta'_v$. We define $g_{uv}=\theta_u+\theta_v+\theta_{uv}$ and $g'_{uv}=\theta'_u+\theta'_v+\theta_{uv}$.  Thus we have $g_{uv} \geq g'_{uv}$. 

\textbf{$\mathcal{U}$ is monotonous:}\newline
Performing update $\gamma_{\mathcal{U}}[g_{uv}] \rightarrow (\theta^{\SU}_u,\theta^{\SU}_v)$ (Eqn.~\cref{equ:uniform-update}) yields

\begin{align}
\theta^{\SU}_u(s) = \theta^{\SU}_v(t):= \underset{s',t'}{\textrm{min}} \ \frac{1}{2}  \  g_{uv}(s',t')
\end{align}

Likewise  $\gamma_{\mathcal{U}}[g'_{uv}] \rightarrow (\theta'^{\SU}_u,\theta'^{\SU}_v)$ yields

\begin{align}
\theta'^{\SU}_u(s) = \theta'^{\SU}_v(t):= \underset{s',t'}{\textrm{min}} \ \frac{1}{2}  \  g'_{uv}(s',t')
\end{align}

As  $g_{uv} \geq g'_{uv}$, we have  $\min_{s't'} \  g_{uv}(s',t') \geq \min_{s't'} \ g'_{uv}(s',t')$. This implies, $\theta^{\SU}_u \geq \theta'^{\SU}_u, \ \ \  \theta^{\SU}_v \geq \theta'^{\SU}_v$. Hence proved.

\vspace{0.5cm}

\textbf{$\SM$ is monotonous:}\newline
Performing update $\gamma_{\mathcal{M}}[g_{uv}] \rightarrow (\theta^{\SM}_u,\theta^{\SM}_v)$ (Eqn.~\ref{equ:MPLP-update}) yields

\begin{gather*}
\theta^{\SM}_u(s) := \frac{1}{2} \ \underset{t'}{\textrm{min}} \ \ g_{uv}(s,t'), \ \ \ 
\theta^{\SM}_v(t) := \frac{1}{2} \ \underset{s'}{\textrm{min}} \ \ g_{uv}(s',t)\\
\end{gather*}

For $\gamma_{\mathcal{M}}[g'_{uv}] \rightarrow (\theta'^{\SM}_u,\theta'^{\SM}_v)$ yields
\begin{gather*}
\theta'^{\SM}_u(s) := \frac{1}{2} \ \underset{t'}{\textrm{min}} \ \ g'_{uv}(s,t'), \ \ \  
\theta'^{\SM}_v(t) := \frac{1}{2} \ \underset{s'}{\textrm{min}} \ \ g'_{uv}(s',t)\\
\end{gather*}

As  $g_{uv} \geq g'_{uv}$, we have $\min_{t'} \ g_{uv}(s,t') \geq \min_{t'} \  g'_{uv}(s,t') \ \ \& \ \  \underset{s'}{\textrm{min}} \ \ g_{uv}(s',t) \geq \underset{s'}{\textrm{min}} \ \  g'_{uv}(s',t)$. This  implies $\theta^{\SM}_u \geq \theta'^{\SM}_u, \ \  \theta^{\SM}_v \geq \theta'^{\SM}_v$. Hence proved.

\vspace{0.5cm}

\textbf{$\mathcal{H}$ is not monotonous:}\newline
To prove that $\mathcal{H}$ is not monotonous we show a counter-example to the monotonous condition stated in theorem~\ref{thm:mplp-is-monotonous}. 

Consider the following unary and pairwise costs, 
\begin{gather*}
\theta_u=\begin{pmatrix}
4 \\
0
\end{pmatrix}  \ \ 
\theta_v=\begin{pmatrix}
2 \\
0
\end{pmatrix}  \ \ 
\theta'_u=\begin{pmatrix}
0 \\
0
\end{pmatrix}  \ \ 
\theta'_v=\begin{pmatrix}
0 \\
0
\end{pmatrix}  \ \ 
\theta_{uv}=\begin{pmatrix}
0 & 1 \\
7 & 5
\end{pmatrix} 
\end{gather*}

These costs satisfy $\theta_u \geq \theta'_u$ and $\theta_v \geq \theta'_v$.


Now, applying the $\H$ operation to $g_{uv}$ and $g'_{uv}$, we get the following potentials

\begin{gather*}
\theta^{\mathcal{H}}_u=\begin{pmatrix}
2.5 \\
2.5
\end{pmatrix}  \ \ 
\theta^{\mathcal{H}}_v=\begin{pmatrix}
3.5 \\
2.5
\end{pmatrix} 
\theta'^{\mathcal{H}}_u=\begin{pmatrix}
0 \\
4
\end{pmatrix}  \ \ 
\theta'^{\mathcal{H}}_v=\begin{pmatrix}
0 \\
1
\end{pmatrix}  \ \ 
\end{gather*}

We thus have $\theta^{\mathcal{H}}_u \ngeq \theta'^{\mathcal{H}}_u$, providing the necessary counter-example. \QED
\end{proof}

\Tdominance*
\begin{proof}
From the definition of dominance we have, if $\gamma$ dominates $\mu$ ($\gamma \geq \mu$) then $\gamma[g_{uv}] \geq \mu[g_{uv}], \ \forall g_{uv}$. 

By the definition of {\em monotonous}, if  $\mu$ is monotonous, $g_{uv}^{1} \geq g_{uv}^{2} \implies \mu[g_{uv}^{1}] \geq \mu[g_{uv}^{2}]$.


Now during the $1^{st}$ iteration we have three cases, 

\begin{itemize}
\item Case \textbf{A}: Nodes $u$ and $v$ of edge $uv$ have not been reparametrized before. 
\item Case \textbf{B}: Nodes $u$ and $v$ of edge $uv$ have been reparametrized before. 
\item Case \textbf{C}: Only node $u$ or $v$ of edge $uv$ have been reparametrized before. 
\end{itemize}

If case $\textbf{A}$ we have by dominance $\gamma[g_{uv}] \geq \mu[g_{uv}]$. 

If case $\textbf{B}$, the unaries $u$ and $v$ have been reparametrized. Let the reparametrized unaries for $\gamma$ be $(\theta^{\gamma}_u,\theta^{\gamma}_v)$ and for $\mu$ be $(\theta^{\mu}_u,\theta^{\mu}_v)$. From $\gamma \geq \mu$ we know $\theta^{\gamma}_u \geq \theta^{\mu}_u$ and $\theta^{\gamma}_v \geq \theta^{\mu}_v$. Then, $g^{\gamma}_{uv}=\theta^{\gamma}_u+\theta^{\gamma}_v+\theta_{uv}$ and $g^{\mu}_{uv}=\theta^{\mu}_u+\theta^{\mu}_v+\theta_{uv}$. It follows $g^{\gamma}_{uv} \geq g^{\mu}_{uv}$. Consider the chain of inequalities

\begin{equation}	
\gamma[g^{\gamma}_{uv}] \geq \mu[g^{\gamma}_{uv}] \geq \mu[g^{\mu}_{uv}]
\end{equation}

The first is true by $\gamma \geq \mu$ dominance. The second is true by $\mu$-monotonicity. Thus for case B also, $\gamma$ results in reparametrized unaries that are co-ordinate wise greater than $\mu$. 

Case $\textbf{C}$, can be proven in much the same way as case $B$.

\QED

\end{proof}

\TconvergenceAC*
\label{sec:proof-arc-consistency}

To prove node-edge agreement we have to show that 

\begin{equation}
\underset{i \rightarrow \infty}{\textrm{lim}} \varepsilon(\mathcal{H}^{i}(\theta))=0
\end{equation}

where $\varepsilon$ is the tolerance factor defined in~\cref{sec:add-notation}.

 By saying that {\tt MPLP++} converges to node-edge agreement, we mean that as the algorithm progresses $\varepsilon$ tends to $0$ and the set of labels belonging to $\mathcal{O}^{\varepsilon}$ is sequentially pruned only to leave an optimal labelling satisfying the node-edge agreement condition, converting $\mathcal{O}^{\varepsilon}$ to $\mathcal{O}^{0}$.

The proof is dependent on several lemmas which are sequentially proved. 

\begin{lemma}
\label{prop:Fcont}
$\mathcal{H}$ is a continuous function.
\end{lemma}

\begin{proof}
Stated differently, we show the proposed reparameterizations in this paper are continuous. To generalize the result across all proposed reparameterizations we recall the idea of an {\em oracle call} from the main paper, \ie operations of the type  $\min_{t \in \Y}g_{uv}(s,t)$,  $\forall s \in \SY$. We show the continuity of the first equation of the {\tt MPLP++} operation only ($1^{st}$ equation of Eqn.~\cref{equ:handshake-update} in the main paper). The other equations of {\tt MPLP++} and other algorithms can be proved similarly.

We prove continuity for only one label of one unary cost (label $s$ and unary $u$ ). Let $\delta>0$. Similar proofs hold for all other unaries and pairwise costs. Consider two edges denoted by the triplet $(\theta_u,\theta_v,\theta_{uv})$ and $(\theta_u',\theta_v,\theta_{uv})$, such that 

\begin{equation} \label{equ:F-cont1}
|\theta_u'(s)-\theta_u(s)| < \delta
\end{equation}

Then, $g_{uv}:=\theta_u+\theta_v+\theta_{uv}$ and $g_{uv}':=\theta_u'+\theta_v+\theta_{uv}$. By~\cref{equ:F-cont1}, we have

\begin{equation}
\forall t \ \  |g_{uv}'(s,t)-g_{uv}(s,t)| < \delta
\end{equation} \label{equ:F-cont2}

Applying the first operation of {\tt MPLP++} to $g_{uv}$ and $g_{uv}'$, we get

\begin{align}
\ThA{u}{H}(s):=\frac{1}{2} \  \underset{t \in \SY}{\textrm{min}} \  g_{uv}(s,t)  \label{equ:Fcont3}\\
\theta'^{\mathcal{H}}_u(s):=\frac{1}{2} \  \underset{t \in \SY}{\textrm{min}} \ g'_{uv}(s,t) \label{equ:Fcont4}
\end{align}

Let $\mid \theta'^{\mathcal{H}}_u(s)-\theta^{\mathcal{H}}_u(s) \mid< \nu$, then by choosing $\delta:=\nu/2$, we have

\begin{equation}
\forall \theta\in \mathbb{R}^{\SI}, \ \forall \nu>0, \exists\delta>0, \ \forall \theta'\in \mathbb{R}^{\SI}, \  \mid \theta_u'(s)-\theta_u(s) \mid < \delta \implies \mid \theta'^{\mathcal{H}}_u(s)-\theta^{\mathcal{H}}_u(s) \mid< \nu
\end{equation} 
\QED

\end{proof}

\begin{lemma}
\label{prop:Dcont}
The dual function $D$ is a continuous function in it's input $\theta$.
\end{lemma}

\begin{proof}
\label{proof:Dcont}
Consider two different reparameterizations $\theta^{\phi},\theta^{\kappa}\in\mathbb{R}^{\SI}$. To prove $D(.)$ is a continuous function, it suffices to show $\forall \nu >0, \exists \delta>0$, such that $|\theta^{\phi} - \theta^{\kappa}|<\delta \implies |D(\theta^{\phi}) - D(\theta^{\kappa})|<\nu$. To do so, we need to recall from ~\cref{sec:add-notation}, $\Theta_{pw}(\phi,\kappa)=\max_{uv\in\SE,st\in\SY^2} |\theta^{\phi}_{uv}(s,t)-\theta^{\kappa}_{uv}(s,t)|$ and $\Theta_{un}(\phi,\kappa)=\max_{u\in\SV,s\in\SY} |\theta^{\phi}_{u}(s)-\theta^{\kappa}_{u}(s)|$. 

We then have
{ \small
\begin{gather*}
\mid D(\theta^\phi)-D(\theta^\kappa)\mid=\mid(\sum_{u \in \mathcal{V}}\underset{s}{\textrm{min}} \theta^{\phi}_u(s)+\sum_{uv \in \mathcal{E}}\underset{(s,t)}{\textrm{min}}\theta^{\phi}_{uv}(s,t))-(\sum_{u \in \mathcal{V}}\underset{s}{\textrm{min}}\theta^{\kappa}_u(s)+\sum_{uv \in \mathcal{E}}\underset{s,t}{\textrm{min}}\theta^{\kappa}_{uv}(s,t))\mid\\
=\mid \sum_{u \in \mathcal{V}}(\underset{s}{\textrm{min}}\theta^{\phi}_u(s)-\underset{s}{\textrm{min}}\theta^{\kappa}_u(s)) \ 
+\sum_{uv \in \mathcal{E}}(\underset{s,t}{\textrm{min}}\theta^{\phi}_{uv}(s,t))-\underset{s,t}{\textrm{min}}\theta^{\kappa}_{uv}(s,t))\mid\\
\leq\mid \sum_{u \in \mathcal{V}}(\underset{s}{\textrm{min}}\theta^{\phi}_u(s)-\underset{s}{\textrm{min}}\theta^{\kappa}_u(s)) \mid  
+ \mid \sum_{uv \in \mathcal{E}}(\underset{s,t}{\textrm{min}}\theta^{\phi}_{uv}(s,t))-\underset{s,t}{\textrm{min}}\theta^{\kappa}_{uv}(s,t))\mid\\
\leq \mid \sum_{u \in \mathcal{V}}\Theta_{un}(\phi,\kappa) \mid + \mid \sum_{uv \in \mathcal{E}}\Theta_{pw}(\phi,\kappa)\mid \ 
\leq \mid \mathcal{V}\mid \mid\Theta_{un}(\phi,\kappa) \mid + \mid \mathcal{E} \mid \mid\Theta_{pw}(\phi,\kappa)\mid\\
\end{gather*}
}
Thus for  $|D(\theta^\phi)-D(\theta^\kappa)| < \nu$, we need to choose a $\delta$ as a function of $|\Theta_{pw}| < \frac{1}{|\mathcal{E}|}$ and $|\Theta_{un}| < \frac{1}{|\mathcal{V}|}$. This can be done as $\H$ is continuous (by proposition~\ref{prop:Fcont}) and pointwise-maxima are also continuous~\cite{boyd2004convex}. \QED

\end{proof}

\begin{lemma}
\label{prop:eps-is-cont}
Tolerance factor $\varepsilon$ is continuous on $\theta$.
\end{lemma}

\begin{proof}
Let $\theta^{\phi}$ and $\theta^{\kappa}$ be two reparameterizations. We prove $\varepsilon$ is continuous for unaries only. The proof for pairwise terms is similar. Let $\Theta_{un}(\phi,\kappa)=\delta$. From the definition of $\Theta_{un}$ we have $\max_{u\in\V,s\in\Y}|\theta_u^{\phi}(s)-\theta_u^{\kappa}(s)|=\delta$. Thus, 

\begin{gather}
|\theta_u^{\phi}(s)-\theta_u^{\kappa}(s)|\leq\delta \label{equ:eps-cont-1}
\end{gather}

~\cref{equ:eps-cont-1} can be rewritten as

\begin{equation}
\label{equ:eps-cont-4}
\theta_u^{\kappa}(s) - \delta \leq \theta_u^{\phi}(s) \leq \theta_u^{\kappa}(s) + \delta
\end{equation}

Now consider the set $\mathcal{O}^{\varepsilon}_u(\theta)=\{s| \theta_u(s) \leq  \min_{s'} (\theta_u(s') +\varepsilon(\theta)) \}$. For a unary in $\mathcal{O}^{\varepsilon}_u(\theta)$

\begin{equation}
\label{equ:eps-cont-5}
\theta^{\phi}_u(s) \leq  \underset{s'}{\textrm{min}} ( \theta^{\phi}_u(s')+\varepsilon(\theta) )
\end{equation}
 
Substituting~\cref{equ:eps-cont-4} in~\cref{equ:eps-cont-5} we get, 

\begin{gather}
\theta^{\kappa}_u(s) - \delta \leq  \underset{s'}{\textrm{min}} ( \theta^{\kappa}_u(s') +\delta +\varepsilon(\theta) ) \implies
\theta^{\kappa}_u(s)  \leq  \underset{s'}{\textrm{min}} ( \theta^{\kappa}_u(s') +\varepsilon(\theta) +2\delta) \label{equ:eps-cont-6}
\end{gather} 

Thus if $s$ satisfies~\cref{equ:eps-cont-5}, it also satisfies~\cref{equ:eps-cont-6}, which has $\varepsilon' \geq \varepsilon + 2\delta$. \QED
 
\end{proof}

\begin{lemma}
\label{prop:monotonic-increase}
 $D(\mathcal{H}^{i+1}(\theta))\geq D(\mathcal{H}^{i}(\theta)),\forall i$, \ie the {\tt MPLP++} reparametrization never decreases the dual $D$.
\end{lemma}
\begin{proof}
Let's consider $D$ to be fixed for all variables except the block $D_{uv}$. Also, we assume that the costs have been normalized, \ie $\min_{s\in\SY}\theta_{u}(s)=0$, $\min_{t\in\SY}\theta_{v}(t)=0$ and $\min_{st\in\SY^2}\theta_{uv}(s,t)=0$. So, we have $D_{uv}(\theta)=0$. We thus have to show that $D_{uv}(\mathcal{H}(\theta))\geq D_{uv}(\theta)=0$. 

The aggregated potential $g_{uv}=\theta_u+\theta_v+\theta_{uv}$ can be written in the form of a $\SY \times\SY$ matrix 

\[
\begin{bmatrix}
r_1+\Delta_{1,1} & \hdots & r_1+\Delta_{1,|\SY|}\\
\vdots & \ddots & \vdots \\
r_{|\SY|}+\Delta_{|\SY|,1} & \hdots & r_{|\SY|}+\Delta_{|\SY|,|\SY|}
\end{bmatrix}
\]

where $r_s$ is the {\em row minimum} and $\Delta_{s,t}\geq 0$. So, $\Delta_{s,t}$ is $0$ for all elements of the row that are equal to $r_s$. As each element of row $s$ of $g_{uv}$ contains $\theta_u(s)$, we know that $r_s\geq \theta_u(s)$, $\forall s \in \Y$.

Now, consider the $1^{st}$ equation of the ${\tt MPLP++}$ operation~\cref{equ:handshake-update}

\begin{align*}
\theta^\H_u(s) & \textstyle :=  \frac{1}{2} \min_{t \in \Y}[g_{uv}(s,t)] = \frac{1}{2}\min_{t \in \Y} r_s+\Delta_{s,t}= \frac{r_s}{2},\ \forall s \in \SY\,,\\
\notag
\theta^\H_v(t) & \textstyle := \frac{1}{2} \min_{s \in \Y}[g_{uv}(s,t)] = \frac{1}{2} \min_{s\in\SY}r_s+\Delta_{s,t}, \  \forall t \in \SY\
\end{align*}

It is $r_s\geq\theta_u(s)\geq 0$ and therefore, $\theta^{\H}_u(s)=r_s/2\geq 0$. Likewise, $\theta^{\H}_v(t)=\dfrac{1}{2}\min_{s\in\Y}[r_s+\Delta_{s,t}]$, where $r_s\geq 0$ and $\Delta_{s,t}\geq 0$ $\forall t$, thus $\theta^{\H}_v(t)\geq 0$. Hence, $D(\H(\theta)) \geq 0$. \QED


\end{proof}

\begin{lemma}
\label{prop:fixedpt}
$\mathcal{H}$ converges to a fixed point.
\end{lemma}

\begin{proof}
Initially, the dual bound $D(\theta^{\phi})$ is computed using~\cref{equ:LP-lower-bound} which takes the min over all unaries $\theta_u(s)$ and pairwise terms $\theta_{uv}(s,t)$ individually. Thus, $D(\theta^{\phi})$ is bounded unless for any $\theta_u(s)$ or $\theta_{uv}(s,t)$ all the elements are $\infty$. If the dual is unbounded, then by strong duality~\cite{boyd2004convex} the primal is also unbounded, and the only energy attainable is $\infty$. On the other hand, if $D(\theta^\phi)$ is bounded, by the LP duality theorem, the primal energy $E(y|\theta)$ serves as an upper bound for the $D(\theta^\phi)$. 

We have proven that the BCA-update $\mathcal{H}$ brings about a monotonic improvement of $D(\theta^\phi)$. As each algorithm involves performing these updates over a sequence of edges $e^{i}$ covering the entire graph, we end up at the end of every iteration with a non-decreasing dual $D(\theta^\phi)$. Thus, each algorithm generates an increasing sequence of $D(\theta^\phi)$ which is bounded from above and by the \emph{Monotone Convergence Theorem} for real numbers $\mathbb{R}$~\cite{bartle2000introduction} (Section 3.3) the algorithm converges. \QED
\end{proof}

\begin{lemma}
\label{prop:D-UB}
$\mathcal{H}$ is bounded \wrt $D$: For any $\theta$ there exists an $M \in \mathbb{R}$ such that $D(\mathcal{H}^{i}(\theta)) < M$ for any $i$.
\end{lemma}

\begin{proof}
\label{proof:D-UB}
As $D$ is a dual LP, we know from LP-Duality Theorem ~\cite{boyd2004convex} that $D(\mathcal{H}^{i}(\theta)) \leq E(y|\theta)$, where $E$ is as in~\cref{equ:energy-min} and $y$ is the labelling. Also, an alternative upper bound can be constructed as follows.
Let $M_u=\max_{s}\theta^{\H}_u(s)$, $M_{uv}=\max_{st} \theta^{\H}_{uv}(s,t)$. We then have $|\theta^{\H}_u(s)| \leq M_u$, $\forall s$ and $|\theta^{\H}_{uv}(s,t)| \leq M_{uv}$ $\forall s, t$. This also implies $\min_{s\in\SY}\theta^{\H}_u(s) \leq M_u$ and $\min_{st\in\SY^2}\theta^{\H}_{uv}(s,t) \leq M_{uv}$.

Thus we have

\begin{gather*}
D(\H(\theta))=\sum_{u \in \mathcal{V}}\underset{s}{\textrm{min}} \ \theta^{\H}_u(s)+\sum_{uv \in \mathcal{E}}\underset{s,t}{\textrm{min}} \ \theta^{\H}_{uv}(s,t) \ 
\leq \sum_{u \in \mathcal{V}}M_u + \sum_{uv \in \mathcal{E}} M_{uv}
\end{gather*}

Therefore $D(\H(\theta))$ is always bounded by $M_{\H(\theta)}=\sum_{u \in \mathcal{V}}M_u + \sum_{uv \in \mathcal{E}} M_{uv}$. We know from~\cref{prop:fixedpt} that $\H$ converges to a fixed point, thus after a certain number of iterations no changes occur in $\H^{i}(\theta)$. To compute the required bound, one must simply take $\max_{i}M_{\H^i(\theta)}$ over all reparameterizations that have occurred. \QED
\end{proof}

\begin{lemma}
\label{prop:F-para-UB}
For any $\theta$ there exists $C>0$ such that $||\mathcal{H}^{i}(\theta)|| \leq C||\theta||$ for any $i$.
\end{lemma}

\begin{proof}
\label{proof:F-para-UB}
We start off by showing that the lemma is true for one label of a unary. Let $\theta^{\phi}(s)$ be the reparameterization of $\theta_u(s)$. Consider the edge triplet $(\theta_u,\theta_v,\theta_{uv})$. Let $m_u:=\min_{s\in\SY}\theta_u(s)$, $M_u:=\max_{s\in\SY}\theta_u(s)$, $M_v:=\max_{t\in\SY}\theta_v(t)$ and $M_{uv}:=\max_{st\in\SY^2}\theta_{uv}(s,t)$.

\begin{align*}
\theta^{\mathcal{H}}_u(s) &:= \frac{1}{2}\underset{t}{\textrm{min}} \  g_{uv}(s,t)\\
\theta^{\mathcal{H}}_u(s) &:= \frac{1}{2}\theta_u(s) + \frac{1}{2} \underset{t}{\textrm{min}} \{ \theta_v(t)+\theta_{uv}(s,t) \}\\
\theta^{\mathcal{H}}_u(s) &\leq \frac{1}{2}\theta_u(s) + \frac{1}{2}(M_v+M_{uv})\\
\theta^{\mathcal{H}}_u(s) &\leq  \left(1 +  \frac{M_v+ M_{uv}}{\theta_u(s)}\right)\theta_u(s) \leq \left(1 +  \frac{M_v+ M_{uv}}{m_u}\right)\theta_u(s)
\end{align*}

Thus, we can choose $C^{i}_u=(1+\frac{M_u+M_{uv}}{m_u})$ and get a bound $||\theta^{\mathcal{H}}_u||\leq C^{i}_u||\theta_u||$. 
As we have to find an upper bound $\forall \theta_u$ and $\forall \theta_{uv}$, we repeat the process and then take the max over all $C^{i}_u$ and $C^{i}_{uv}$ obtaining $C^{i}$. This gives us a $C^{i}$ that satisfies $||\theta^{i+1}=\mathcal{H}(\theta^{i})|| \leq C^{i}||\theta^{i}||$. By \cref{prop:fixedpt} after a certain number of iterations $n$, $\theta^{n}$ converges, so to get the required result we simply simply take $C=\prod_{i=1}^{n}C^{i}$, giving $||\H^{i}(\theta)||\leq C||\theta||$ for any $i$. \QED

\end{proof}

\begin{lemma}
\label{prop:OPT-shrinks}
$D(\mathcal{H}(\theta))=D(\theta)$ implies $\mathcal{O}^{0}(\mathcal{H}(\theta)) \subseteq \mathcal{O}^{0}(\theta)$, \ie at a fixed point no new optimal labellings are achieved. If additionally  $\varepsilon(\theta^{\H})>0$, then $\mathcal{O}^{0}(\mathcal{H}(\theta)) \subset \mathcal{O}^{0}(\theta)$, \ie the inequality is  strict.
\end{lemma}

\begin{proof}
\label{proof:OPT-shrinks}
Let $\theta^{\H}=\H(\theta)$, \ie $\theta^{\H}$ is the reparameterized cost vector obtained after reparameterizing the original cost vector $\theta$. We prove the first assertion for unary cost $u$ only, the remaining costs can be proved similarly. To do so, we assume that only after reparameterization only the optimal label of unary cost $u$ has changed.

Let this label be $s^{*}=\argmin_{s}\theta^{\H}_u(s)$. This implies $s^{*} \in \mathcal{O}(\theta^{\H})$. Now, to prove the first assertion, we have to show the inclusion $s^{*} \in \mathcal{O}^{0}_u(\theta^{\H}) \implies s^{*} \in \mathcal{O}^{0}_u(\theta)$. 

Let us additionally define a function $\eta$ that takes as input a node index $p$ and outputs its optimal labelling, \ie $k=\eta(v)$ implies $k=\argmin_{s \in \Y}\theta_p(s)$. We define this function over all unary costs except $u$.  As $\theta^{\H}$ and $\theta$ differ in only the label $u$, we can use $\eta$ for choosing the optimal labelling for both of them. 

Since, $D(\theta^{\H})=D(\theta)$, we have

\begin{equation} \label{equ:repara-opt-shrinks}
\sum_{v \in \V} \min_{s \in \Y} \theta^{\H}_v(s) + \sum_{vw \in \E} \min_{st \in \Y^2} \theta^{\H}_{vw}(s,t)=\sum_{v \in \V} \min_{s \in \Y} \theta_v(s)+ \sum_{vw \in \E} \min_{st \in \Y^2} \theta_{vw}(s,t)
\end{equation}

Since all the labels but the ones for $u$ are the same, the terms on both sides of~\cref{equ:repara-opt-shrinks} cancel out leaving

\begin{equation} \label{equ:repara-opt-shrinks2}
\min_{s \in \Y} \theta^{\H}_u(s) + \sum_{ uv \mid v \in Nb(u)} \min_{st \in \Y^2} \theta^{\H}_{uv}(s,t)=\min_{s \in \Y} \theta_u(s)+ \sum_{uv \mid v \in Nb(u)} \min_{st \in \Y^2} \theta_{uv}(s,t)
\end{equation}

Substituting $s^{*}=\argmin_{s}\theta^{\H}_u(s)$ and $\eta$ into \cref{equ:repara-opt-shrinks2} we get

\begin{equation} \label{equ:repara-opt-shrinks3}
\theta^{\H}_u(s^{*}) + \sum_{uv \mid v \in Nb(u)}  \theta^{\H}_{uv}(s^{*},\eta(v))=\min_{s \in \Y} \theta_u(s)+ \sum_{uv \mid v \in Nb(u)}  \theta_{uv}(s,\eta(v))
\end{equation}

Now, by the definition of reparameterization we have,  $\forall s \in \Y$

\begin{equation} \label{equ:repara-opt-shrinks4}
\theta^{\H}_u(s) + \sum_{uv \mid v \in Nb(u)} \theta^{\H}_{uv}(s,\eta(v))= \theta_u(s)+ \sum_{uv \mid v \in Nb(u)} \theta_{uv}(s,\eta(v))
\end{equation}

This also holds true for label $s^{*}$, thus

\begin{equation} \label{equ:repara-opt-shrinks5}
\theta^{\H}_u(s^{*}) + \sum_{uv \mid v \in Nb(u)} \theta^{\H}_{uv}(s^{*},\eta(v))= \theta_u(s^{*})+ \sum_{uv \mid v \in Nb(u)} \theta_{uv}(s^{*},\eta(v))
\end{equation}

Now, equating the RHS of~\cref{equ:repara-opt-shrinks5} and the RHS of~\cref{equ:repara-opt-shrinks3}, we get 

\begin{equation} \label{equ:repara-opt-shrinks6}
\theta_u(s^{*})+ \sum_{uv \mid v \in Nb(u)} \theta_{uv}(s^{*},\eta(v))=\min_{s \in \Y} \theta_u(s)+ \sum_{uv \mid v \in Nb(u)}  \theta_{uv}(s,\eta(v))
\end{equation}

Thus $s^{*}=\argmin_{s\in\Y}\theta_u(s) \implies s^{*}\in\mathcal{O}^{0}_u(\theta)$, proving the first assertion.

To prove the second assertion, we have to show that if $\epsilon(\theta^{\H})>0$ and $D(\H(\theta))=D(\theta)$, there exists $s'\in\mathcal{O}^{0}_u(\theta)$ such that $s'\notin\mathcal{O}^{0}_u(\theta^{\H})$.

Since, $s' \in \mathcal{O}^{0}_u(\theta)$, we have
\begin{equation}
\theta_u(s')+\sum_{v \in Nb(u)}\theta_{uv}(s',\eta(v)) = \min_{s\in\SY}\theta_u(s)+\sum_{v \in Nb(u)} \min_{s\in\Y}\theta_{uv}(s,\eta(v))
\end{equation}

From the reparameterization relation we have

\begin{equation}
\theta_u(s')+\sum_{v \in Nb(u)}\theta_{uv}(s',\eta(v)) = \theta_u^{\H}(s')+\sum_{v \in Nb(u)}\theta_{uv}^{\H}(s',\eta(v))
\end{equation}

Since $\varepsilon(\theta^{\H})>0$, it has to be the case that 

\begin{equation}
\theta_u^{\H}(s')+\sum_{v \in Nb(u)}\theta_{uv}^{\H}(s',\eta(v))>\min_{s\in\Y}\theta_u^{\H}(s)+\sum_{v \in Nb(u)}\min_{s\in\Y}\theta_{uv}^{\H}(s,\eta(v))
\end{equation}

otherwise $\varepsilon(\theta^{\H})=0$. Thus, $s'\notin\mathcal{O}^{0}_u(\theta^{\H})$.

 
%
%
%
%

\end{proof}

\begin{lemma}
\label{prop:OPT-stabilizes}
There exists an $n$, such that $D(\mathcal{H}^{n+m}(\theta))=D(\mathcal{H}^{n}(\theta))$ implies \\
 ~$\mathcal{O}^{0}(\mathcal{H}^{n+m}(\theta)) = \mathcal{O}^{0}(\mathcal{H}^{n}(\theta)),\forall m \geq 0$.
\end{lemma}

\begin{proof}
\label{proof:OPT-stabilizes}
Let $\theta^{0},\theta^{1},...$ be the sequence of vectors generated from $\mathcal{H}$ via $\theta^{i+1}=\mathcal{H}(\theta^{i})$. After a certain number of iterations, we have to show that $\mathcal{O}^{0}(\theta^{n+m})=\mathcal{O}^{0}(\theta^{n})$, for all $m \geq0$. From \cref{prop:OPT-shrinks} we have after a fixed point has been reached at iteration $i$, $\mathcal{O}^{0}(\theta^{i}) \supset \mathcal{O}^{0}(\theta^{i+1}) \supset \mathcal{O}^{0}(\theta^{i+2}) \supset ...$. Since $\mathcal{O}^{0}(\theta)$ is finite, it cannot shrink indefinitely and after a certain number ($n$) of iterations $\mathcal{O}^{0}(\theta^{n})$ will become consistent.  Yielding, $\mathcal{O}^{0}(\theta^{m+n})=\mathcal{O}^{0}(\theta^{n})$, $\forall m \geq 0$. \QED
\end{proof}

Now, we are ready to prove ~\cref{thm:alg-convergence}

Combining the above lemmas, we introduce the notion of {\em consistency-enforcing} algorithms. An algorithm is consistency-enforcing if it satisfies Lemmas \ref{prop:Fcont}, \ref{prop:Dcont},  \ref{prop:eps-is-cont}, \ref{prop:monotonic-increase}, \ref{prop:fixedpt}, \ref{prop:D-UB}, \ref{prop:F-para-UB}, \ref{prop:OPT-shrinks},  \ref{prop:OPT-stabilizes},  .

The {\tt MPLP++} operator $\mathcal{H}$ is consistency enforcing then $\lim_{i  \rightarrow \infty} \varepsilon(\mathcal{H}^{i}(\theta))=0.$

\begin{equation}
\underset{i \rightarrow \infty}{\textrm{lim}} \varepsilon(\mathcal{H}^{i}(\theta))=0
\end{equation}

\TconvergenceAC*

\begin{proof}
By virtue of~\cref{prop:F-para-UB} and ~\cref{prop:monotonic-increase}, the sequence $\theta^{i}=\mathcal{H}^{i}(\theta)$ is bounded. Therefore, by the Bolzano-Weierstrass Theorem \cite{bartle2000introduction}, there exists a converging subsequence $\theta^{i(j)}$, $j=1,2,\hdots,$ where $j > j'$ implies $i(j) > i(j')$, \ie the limit $\theta^{*}:=\lim_{j \rightarrow \infty}\theta^{i(j)}$ exists. Let us show that it holds 

\begin{equation}	\label{equ:epsilon-tends-0}
\varepsilon(\theta^{*})=0
\end{equation}

for any converging subsequence of $\theta^{i}$.

Since due to convergence of $\mathcal{H}$ (shown in \cref{prop:monotonic-increase}) and \cref{prop:D-UB} the sequence $D(\theta^{i})$ is non-decreasing and bounded from above, and therefore converges to a limit point $D^{*}:=\underset{i \rightarrow \infty}{\textrm{lim}}D(\theta^{i})$. Therefore, it also holds 

\begin{equation}
D^{*}= \underset{j \rightarrow \infty}{\textrm{lim}} D(\theta^{i(j)}) = \underset{j \rightarrow \infty}{\textrm{lim}} D(\theta^{i(j)+n}), \ \ \ \forall n \geq 0
\end{equation}

This implies 

\begin{equation}
0 = \underset{j \rightarrow \infty}{\textrm{lim}} D(\theta^{i(j)}) - \underset{j \rightarrow \infty}{\textrm{lim}} D(\theta^{i(j)+n}) = \underset{j \rightarrow \infty}{\textrm{lim}} D(\theta^{i(j)}) -  D(\H^{n}(\theta^{i(j)})).
\end{equation}

Since $D$ is continuous it holds

\begin{equation}
0 = \underset{j \rightarrow \infty}{\textrm{lim}}( D(\theta^{i(j)}) -  D(\H^{n}(\theta^{i(j)})) ).
\end{equation}

and therefore,~\cref{equ:epsilon-tends-0} holds by virtue of~\cref{prop:OPT-stabilizes}.

Since, $\varepsilon$ is a continuous function,~\cref{equ:epsilon-tends-0} implies

\begin{equation}	\label{equ:eps-to-0}
\underset{j \rightarrow \infty}{\textrm{lim}} \varepsilon(\theta^{i(j)})=0
\end{equation}

for any converging sub-sequence $\theta^{i(j)}$.

Now, considering the sequence $\varepsilon(\theta^{i})$, we know by virtue of~\cref{prop:Fcont} $s^{i}:=\sup_{j \geq i} \varepsilon(\theta^{j})$. Sequence $s^{i}$ is a monotonically non-increasing sequence of non-negative numbers and therefore it has a limit $s^{*} = \underset{i \rightarrow \infty}{\textrm{lim}} s^{i}$.

According to the ``Theorem of Superior and Inferior Limits''~\cite{bartle2000introduction} there exists  a subsequence $\varepsilon(\theta^{i{'}(j)})$ such that

\begin{equation}
\underset{j \rightarrow \infty}{\textrm{lim}} \varepsilon(\theta^{i{'}(j(k))})=s^{*}
\end{equation}

The sequence $\theta^{i'(j)}$ is bounded virtue of~\cref{prop:F-para-UB} and therefore contains a converging subsequence $\theta^{i'(j(k))}$. For this subsequence it also holds,

\begin{equation}
\underset{k \rightarrow \infty}{\textrm{lim}} \varepsilon(\theta^{i'(j(k))}) = s^{*}
\end{equation}

%

At the same time, as proved in~\cref{equ:eps-to-0}, for any converging subsequence it holds

\begin{equation}
0 = \underset{k \rightarrow \infty}{\textrm{lim}} \varepsilon(\theta^{i'(j(k))}) = s^{*} = \underset{i \rightarrow \infty}{\textrm{lim}} \underset{k \geq i}{\textrm{sup}} \ \varepsilon(\theta^{k})
\end{equation}

Finally, $0 \leq \varepsilon(\theta^{i}) \leq \underset{k \geq i}{\textrm{sup}} \  \varepsilon(\theta^{k})$ implies $\underset{i \rightarrow \infty}{\textrm{lim}} \varepsilon(\theta^{i})=0$. \QED
\end{proof}


%

\TblockOpt*
\begin{proof}
The if clause has been proven in (\cite{schlesinger1976syntactic}, Theorem 2).
The only if clause can be proven as follows:

As mentioned in the main paper and shown in~\cite{schlesingera2011diffusion,werner2007linear} the Dual LP $D(\phi)$ is a concave, piecewise linear function. 

For proving optimality of block $D_{uv}(\phi_{u \leftrightarrow v})$ we need to show $\mathbf{0} \in \partial D_{uv}(\phi_{u \leftrightarrow v})$, where $\partial D_{uv}(\phi_{u \leftrightarrow v})$ is the super-differential of $D_{uv}(\phi_{u \leftrightarrow v})$. Reconsidering the dual $D_{uv}(\phi_{u \leftrightarrow v})$ 

\begin{equation}
D_{uv}(\phi_{u \leftrightarrow v}):= \min\limits_{st\in\SY^2}\theta^{\phi}_{uv}(s,t)+ \min\limits_{s\in\SY}\theta^{\phi}_{u}(s)+\min\limits_{t\in\SY}\theta^{\phi}_{v}(t)\,,  
\end{equation}

Let $x'_u=\argmin_{s\in\Y}\theta^{\phi}_{u}(s)$, $x'_v=\argmin_{t\in\Y}\theta^{\phi}_{v}(t)$ and \\ $(x''_u,x''_v)=\argmin_{st \in \Y^2}\theta^{\phi}_{uv}(s,t)$. Taking the super-gradient of $D_{uv}(\phi_{u \leftrightarrow v})$ \wrt $\phi_{u \rightarrow v}$, we have

\begin{align}
\frac{\partial D_{uv}(\phi_{u \leftrightarrow v})}{\partial \phi_{u \rightarrow v}(s)} \colon=\begin{cases}
	0,  \ \ \ &s \neq x'_u, \ s \neq x''_u\\
	0,  \ \ \ &s = x'_u, \ s = x''_u\\
	1, \ \ \ &s = x'_u, \ s \neq x''_u\\
	-1,  \ \ \ &s \neq x'_u, \  s = x''_u\\
                    \end{cases}
\end{align}

Likewise, taking the super-gradient of $D_{uv}(\phi_{u \leftrightarrow v})$ \wrt $\phi_{v \rightarrow u}$, we have

\begin{align}
\frac{\partial D_{uv}(\phi_{u \leftrightarrow v})}{\partial \phi_{v \rightarrow u}(t)} \colon=\begin{cases}
	0,  \ \ \ &t \neq x'_v, \ t \neq x''_v\\
	0,  \ \ \ &t = x'_v, \ t = x''_v\\
	1, \ \ \ &t = x'_v, \ t \neq x''_v\\
	-1,  \ \ \ &t \neq x'_v, \  t = x''_v\\
                    \end{cases}
\end{align}

Thus, if $s=x'_u=x''_u$ and $t=x'_v=x''_v$, we have $0 \in \frac{\partial D_{uv}({\phi_{v \leftrightarrow u}})}{\partial \phi_{v \leftrightarrow u}(s,t)}$, proving the only if clause. \QED
\end{proof}


\end{document}